%% file: main.tex
\documentclass{article}

\usepackage{microtype}
\usepackage{graphicx}
\usepackage{subfigure}
\usepackage{booktabs} % for professional tables

\usepackage{hyperref}

\usepackage[accepted]{icml2023}

\usepackage{amsmath}
\usepackage{amssymb}
\usepackage{mathtools}
\usepackage{amsthm}

\usepackage[capitalize,noabbrev]{cleveref}

\theoremstyle{plain}
\newtheorem{theorem}{Theorem}[section]

\newtheorem{lemma}[theorem]{Lemma}

\theoremstyle{definition}

\theoremstyle{remark}

\usepackage[textsize=tiny]{todonotes}

\icmltitlerunning{Trajectory-Aware Eligibility Traces}

\usepackage{paper_style}

\begin{document}

\twocolumn[
\icmltitle{Trajectory-Aware Eligibility Traces for\texorpdfstring{\\}{ }Off-Policy Reinforcement Learning}



\icmlsetsymbol{equal}{*}

\begin{icmlauthorlist}
\icmlauthor{Brett Daley}{uofa,amii}
\icmlauthor{Martha White}{uofa,amii,cifar}
\icmlauthor{Christopher Amato}{neu}
\icmlauthor{Marlos C.~Machado}{uofa,amii,cifar}
\end{icmlauthorlist}

\icmlaffiliation{uofa}{Department of Computing Science, University of Alberta, Edmonton, AB, Canada}
\icmlaffiliation{amii}{Alberta Machine Intelligence Institute}
\icmlaffiliation{cifar}{Canada CIFAR AI Chair}
\icmlaffiliation{neu}{Khoury College of Computer Sciences, Northeastern University, Boston, MA, USA}

\icmlcorrespondingauthor{Brett Daley}{brett.daley@ualberta.ca}

\icmlkeywords{Reinforcement Learning, Temporal-Difference Learning, Off-Policy, Eligibility Traces}

\vskip 0.3in
]



\printAffiliationsAndNotice{}  % leave blank if no need to mention equal contribution

\input{paper_body}

{\fontsize{9.5pt}{\baselineskip} \selectfont
\bibliography{main}
}
\bibliographystyle{icml2023}

%%%%%%%%%%%%%%%%%%%%%%%%%%%%%%%%%%%%%%%%%%%%%%%%%%%%%%%%%%%%%%%%%%%%%%%%%%%%%%%
%%%%%%%%%%%%%%%%%%%%%%%%%%%%%%%%%%%%%%%%%%%%%%%%%%%%%%%%%%%%%%%%%%%%%%%%%%%%%%%
% APPENDIX
%%%%%%%%%%%%%%%%%%%%%%%%%%%%%%%%%%%%%%%%%%%%%%%%%%%%%%%%%%%%%%%%%%%%%%%%%%%%%%%
%%%%%%%%%%%%%%%%%%%%%%%%%%%%%%%%%%%%%%%%%%%%%%%%%%%%%%%%%%%%%%%%%%%%%%%%%%%%%%%
\newpage
\appendix
\onecolumn
\input{paper_appendix}

%%%%%%%%%%%%%%%%%%%%%%%%%%%%%%%%%%%%%%%%%%%%%%%%%%%%%%%%%%%%%%%%%%%%%%%%%%%%%%%
%%%%%%%%%%%%%%%%%%%%%%%%%%%%%%%%%%%%%%%%%%%%%%%%%%%%%%%%%%%%%%%%%%%%%%%%%%%%%%%

\end{document}

%% file: paper_body.tex
\begin{abstract}
    Off-policy learning from multistep returns is crucial for sample-efficient reinforcement learning, but counteracting off-policy bias without exacerbating variance is challenging.
    Classically, off-policy bias is corrected in a \emph{per-decision} manner:
    past temporal-difference errors are re-weighted by the instantaneous Importance Sampling (IS) ratio after each action via eligibility traces.
    Many off-policy algorithms rely on this mechanism, along with differing protocols for \emph{cutting} the IS ratios to combat the variance of the IS estimator.
    Unfortunately, once a trace has been fully cut, the effect cannot be reversed.
    This has led to the development of credit-assignment strategies that account for multiple past experiences at a time.
    These \emph{trajectory-aware} methods have not been extensively analyzed, and their theoretical justification remains uncertain.
    In this paper, we propose a multistep operator that can express both per-decision and trajectory-aware methods.
    We prove convergence conditions for our operator in the tabular setting, establishing the first guarantees for several existing methods as well as many new ones.
    Finally, we introduce Recency-Bounded Importance Sampling (RBIS), which leverages trajectory awareness to perform robustly across $\lambda$-values in several off-policy control tasks.
\end{abstract}

\section{Introduction}

Reinforcement learning concerns an agent interacting with its environment through trial and error to maximize its expected cumulative reward.
One of the great challenges of reinforcement learning is the temporal credit assignment problem \citep{sutton1984temporal}:
upon receiving a reward, which past actions should be held responsible and, hence, be reinforced?
Basic temporal-difference (TD) methods assign credit to the immediately taken action \citep[e.g.,][]{watkins1989learning,rummery1994line}, bootstrapping from previous experience to learn long-term dependencies.
This process requires a large number of repetitions to generate effective behaviors from rewards, motivating research into \textit{multistep} return estimation
in which credit is distributed among multiple past actions according to some eligibility rule \citep[e.g.,][]{sutton1988learning}.
\looseness=-1

One challenge of multistep estimators is that they generally have higher variance than 1-step estimators~\citep{kearns2000bias}.
This is exacerbated in the off-policy setting, where environment interaction is conducted according to a behavior policy that differs from the target policy for which returns are being estimated.
The discrepancy between the two policies manifests mathematically as bias in the return estimation, which can be detrimental to learning if left unaddressed \citep{precup2000eligibility}.
Despite these challenges, off-policy learning is important for exploration and sample efficiency.
The canonical bias-correction technique is Importance Sampling \citep[IS;][]{kahn1951estimation}, wherein the bias due to the differing policies is eliminated by the product of their probability ratios \citep{precup2000eligibility}.
Although IS theoretically resolves the off-policy bias, it can suffer from extreme variance that makes it largely impractical.
\looseness=-1

Directly managing the variance of the IS estimator has been a fruitful avenue for developing efficient off-policy algorithms.
Past work has focused on modifying the individual IS ratios to reduce the variance of the full update:
e.g., Tree Backup \citep{precup2000eligibility}, $\text{Q}^\pi$($\lambda$) \citep{harutyunyan2016q}, Retrace \citep{munos2016safe}, ABQ \citep{mahmood2017multi}, and C-trace \citep{rowland2020adaptive}.
All of these methods can be implemented online with \textit{per-decision} rules \citep{precup2000eligibility} that determine how much to reduce, or \emph{cut}, the IS ratio according to the current state-action pair.
The re-weighted TD error is then broadcast to previous experiences using eligibility traces \citep{barto1983neuronlike, sutton1984temporal}.
The decisions made by these algorithms are Markov in the sense that each iterative off-policy correction depends on only the current state-action pair.
One issue with this is that it can lead to suboptimal decisions, since fully cutting a trace cannot be reversed later.
In contrast, a \textit{trajectory-aware} method can examine an entire sequence of past state-action pairs to make globally better decisions regarding credit assignment;
for example, when a specific transition yields a high IS ratio, a trajectory-aware method can choose to not cut the trace if the product of all previous IS ratios remains small.

Indeed, some existing off-policy methods already conduct offline bias correction in a trajectory-aware manner.
Perhaps the simplest example is Truncated IS, where the IS ratio products are pre-calculated offline and then clipped to some finite value (see \Cref{sect:example_algorithms}).
More recently, \citet{munos2016safe} suggested a recursive variant of Retrace that automatically relaxes the clipping bound when its historical trace magnitude becomes small;
the authors conjectured that this could lead to faster learning.
No theoretical analysis has been conducted on trajectory-aware algorithms such as these;
their convergence properties are unknown, and the space of possible algorithms has not yet been fully explored.
\looseness=-1

To better understand these algorithms, and to support new discoveries of efficient algorithms, we introduce a unifying theoretical perspective on per-decision and trajectory-aware off-policy corrections.
We propose a multistep operator that accounts for arbitrary dependencies on past experiences, significantly generalizing the per-decision $\mathcal{R}$ operator introduced by \citet{munos2016safe}.
We prove that our operator converges for policy evaluation and control.
In the latter case, we remove the assumptions of increasingly greedy policies and pessimistic initialization used by \citet{munos2016safe}, which has implications for per-decision methods.
Finally, we derive a new method from our theory, Recency-Bounded Importance Sampling (RBIS), which performs favorably to other trajectory-aware methods across a wide range of $\lambda$-values in an off-policy control task.

\section{Preliminaries}

We consider Markov Decision Processes (MDPs) of the form
$(\mathcal{S}, \mathcal{A}, P, R, \gamma)$.
$\mathcal{S}$ and $\mathcal{A}$ are finite sets of states and actions, respectively.
Letting $\Delta \mathcal{X}$ denote the set of distributions over a set $\mathcal{X}$, then
${P \colon \mathcal{S} \times \mathcal{A} \to \Delta \mathcal{S}}$
is the transition function,
$R \colon \mathcal{S} \times \mathcal{A} \to \mathbb{R}$
is the reward function, and $\gamma \in [0,1)$ is the discount factor.
A policy $\pi \colon \mathcal{S} \to \Delta \mathcal{A}$ determines an agent's probability of selecting a given action in each state.
A value function $Q \colon \mathcal{S} \times \mathcal{A} \to \mathbb{R}$ represents the agent's estimate of the expected return achievable from each state-action pair.
For a policy $\pi$, we define the operator
\begin{equation*}
    (P_\pi Q)(s,a) \coloneqq \sum_{s' \in \mathcal{S}} \sum_{a' \in \mathcal{A}} P(s'|s,a) \pi(a'|s') Q(s',a')
    .
\end{equation*}
As a shorthand, we represent value functions and the reward function as vectors in $\mathbb{R}^n$, where ${n = |\mathcal{S} \times \mathcal{A}|}$.
Linear operators such as $P_\pi$ can hence be interpreted as $n \times n$ square matrices that multiply these vectors, with repeated application corresponding to exponentiation:
$P_\pi^t Q = P_\pi (P_\pi^{t-1} Q)$.
\looseness=-1

In the \textit{policy evaluation} setting, we seek to estimate the expected discounted returns for policy $\pi$, given by
$Q^\pi \coloneqq \sum_{t=0}^\infty \gamma^t P_\pi^t R$.
The value function $Q^\pi$ is the unique fixed point of the Bellman operator $T_\pi Q \coloneqq R + \gamma P_\pi Q$,
i.e., it uniquely solves the Bellman equation
$T_\pi Q^\pi = Q^\pi$ \citep{bellman1966dynamic}.
In the \textit{control} setting, we seek to estimate the expected returns $Q^*$ under the optimal policy $\pi^*$.
$Q^*$ is the unique fixed point of the Bellman optimality operator
$(T Q)(s,a) \coloneqq \max_\pi\ (T_\pi Q)(s,a)$,
i.e., it uniquely solves the Bellman optimality equation $T Q^* = Q^*$.
We are particularly interested in the \textit{off-policy} learning case, where trajectories of the form
$(S_0, A_0), (S_1, A_1), (S_2, A_2), \dots$
are generated by interacting with the MDP using a behavior policy $\mu$, where $\mu \neq \pi$.
We define the TD error for policy $\pi$ at time $t$ as
\begin{equation*}
    \delta^\pi_t \coloneqq R_t + \gamma \sum_{a' \in \mathcal{A}} \pi(a'|S_{t+1}) Q(S_{t+1}, a') - Q(S_t, A_t)
    ,
\end{equation*}
where $R_t \coloneqq R(S_t,A_t)$.
Let $\smash{\rho_k \coloneqq \frac{\pi(A_k|S_k)}{\mu(A_k|S_k)}}$ for brevity.
\citet{munos2016safe} introduced the off-policy operator
\begin{align}
    \nonumber
    &(\mathcal{R}Q)(s,a) \coloneqq Q(s,a)~+ \\*
    \label{eq:R_op_def}
    &\qquad \expect{\Bigg}{\mu}{\sum_{t=0}^\infty \gamma^t \Bigg(\prod_{k=1}^t c_k\Bigg) \delta^\pi_t}{(S_0,A_0)=(s,a)}
    ,
\end{align}
where $c_k \coloneqq c(S_k,A_k) \in [0, \rho_k]$.
We refer to the product $\prod_{k=1}^t c_k$ as the \emph{trace} for $(s,a)$ at time $t$.
If any $c_k < \rho_k$, we say that the trace has been (partially) cut.
If any $c_k = 0$, then we have fully cut it.
If the trace is fully cut at $t=1$, i.e., $c_1=0$, then
$(\mathcal{R}Q)(s,a) = Q(s,a) + \E[\delta^\pi_0 \! \! \mid \! \! (S_0,A_0) = (s,a)] = R(s,a) + \gamma \E[\sum_{a' \in \mathcal{A}} \pi(a'|S_1) Q(S_1, a') \! \! \mid \! \! (S_0,A_0) = (s,a)]$,
which is the standard 1-step bootstrap target like in TD(0) \citep{sutton1988learning}.
Notice that each $c_k$ is Markov, as it depends only on $(S_k,A_k)$ and is otherwise independent of the preceding trajectory.
In other words, the update for $\mathcal{R}$ can be calculated \textit{per decision} \citep{precup2000eligibility}, permitting an efficient online implementation with eligibility traces.

\section{Trajectory-Aware Eligibility Traces}

While per-decision traces are convenient from a computational perspective, they require making choices about how much to cut the trace without considering the effects of previous choices.
This can lead to suboptimal decisions;
for example, if the trace is cut by setting $c_k = 0$ at some timestep, then the effect cannot be reversed later.
Regardless of whatever new experiences are encountered by the agent, experiences before time $k$ will be ineligible for credit assignment, resulting in an opportunity cost.
In fact, this exact phenomenon is why Watkins' Q($\lambda$) \citep{watkins1989learning} often learns more slowly than Peng's Q($\lambda$) \citep{peng1996incremental}, even though the former avoids off-policy bias \citep{sutton1998reinforcement, daley2019reconciling, kozuno2021revisiting}.
The same effect (but to a lesser extent) impacts Tree Backup and Retrace, where $c_k \leq 1$ always in \Cref{eq:R_op_def}, implying that the traces for past experiences can never increase.

\input{figures/fig_tightrope}

We illustrate this phenomenon in a small, deterministic MDP that we call the Tightrope Problem (see \Cref{fig:tightrope}).
The environment consists of $n$ sequential, non-terminal states with two actions $a_1,a_2$ available.
The agent starts in state $s_1$ and advances from $s_i$ to $s_{i+1}$ whenever it takes action $a_1$.
If $i=n$, then the episode terminates and the agent receives $+1$ reward.
Taking action $a_2$ in any state immediately terminates the episode with no reward.
Clearly, the optimal policy is to execute $a_1$ regardless of the state.

Now consider the following off-policy learning scenario.
Suppose the agent's behavior policy $\mu$ is uniform random, but the target policy $\pi$ is $\epsilon$-greedy with respect to a value function $Q$.
For each state $s$, it follows that $\pi(a|s) = 1-\epsilon$ if $a = \argmax_{a'} Q(s,a')$ and $\pi(a|s) = \epsilon$ otherwise.
We assume $\epsilon$ is small in the sense that $\epsilon < \frac{1}{2}$, and that ${\gamma = 1}$.
Suppose now that the agent successfully receives the $+1$ reward during an episode, implying that it took action $a_1$ on every timestep.
We can compute the eligibility of the initial state-action pair $(s_1,a_1)$ as an expression in the number $k$ of incorrect actions in the greedy policy (i.e., where $\argmax_{a'} Q(s,a') \neq a_1$) out of the $n-1$ subsequent actions.
Letting $\lambda \in [0,1]$ be a decay parameter, the standard IS estimator (which does not cut traces when $\lambda=1$) provides an eligibility of
\begin{equation}
    \label{eq:tightrope_is}
    \left(\lambda \frac{1-\epsilon}{1 \mathbin{/} 2}\right)^{n-1-k} \left(\lambda \frac{\epsilon}{1 \mathbin{/} 2}\right)^k
    = \lambda^{n-1} [2(1-\epsilon)]^{n-1-k} (2\epsilon)^k
    .
\end{equation}
This value can be greater than $1$ when $k \ll n - 1$, which suggests that the agent's behavior should be heavily reinforced when the greedy policy agrees closely with the optimal policy;
however, a per-decision method like Retrace, which cuts traces without considering the full trajectory (see \Cref{sect:example_algorithms}), ultimately assigns a much lower eligibility:
\begin{equation*}
    \left[\lambda \min \! \left(\! 1,~\frac{1-\epsilon}{1 \mathbin{/} 2} \! \right) \! \right]^{n-1-k}
    \! \left[\lambda \min \! \left(\! 1,~\frac{\epsilon}{1 \mathbin{/} 2} \! \right) \! \right]^k
    \! = \lambda^{n-1} (2\epsilon)^k
    .
\end{equation*}
The eligibility now decays monotonically for every suboptimal action in the greedy policy, illuminating how per-decision trace cutting can lead to excessively small eligibilities, especially when $\epsilon$ is close to $0$.

This issue stems from the fact that Retrace is not \emph{aware} of its past eligibilities, and continues to decay them even when they already form an underestimate compared to IS.
This issue is not unique to Retrace, and affects other per-decision methods like Tree Backup.
Instead, a \emph{trajectory-aware} method that can actively adapt its trace-cutting behavior based on the magnitude of past eligibilities would be better.

One way to obtain a trajectory-aware method is to compute the exact IS product in \Cref{eq:tightrope_is}, and then make adjustments to it to achieve certain properties (e.g., convergence and variance reduction).
For example, Truncated IS (see \Cref{sect:example_algorithms}) simply imposes a fixed bound on the IS estimator:
\begin{equation}
    \lambda^{n-1} \min \! \left(1,~[2(1-\epsilon)]^{n-1-k} (2\epsilon)^k\right)
    .
\end{equation}
Ignoring $\lambda$, Truncated IS reduces the eligibility only when it exceeds a pre-specified threshold, effectively avoiding trace cuts when the true IS estimate is small.
In \Cref{sect:rbis}, we propose an algorithm, RBIS, which achieves a similar effect using a recursive, time-decaying threshold.

As this example demonstrates, it can be advantageous to consider the agent's past experiences to produce better decisions regarding credit assignment.
One of our principal contributions is the proposal and analysis of an off-policy operator $\M$ that encompasses this possibility.
Let $\F_t \coloneqq (S_0, A_0), (S_1, A_1), \dots, (S_t, A_t)$.
We define $\M$ such that
\looseness=-1
\begin{align}
    \nonumber
    &(\M Q)(s,a) \coloneqq Q(s,a)~+ \\*
    \label{eq:M_op_def}
    &\qquad \expect{\Bigg}{\mu}{\sum_{t=0}^\infty \gamma^t \beta_t \delta^\pi_t}{(S_0,A_0)=(s,a)}
    ,
\end{align}
where $\beta_t \coloneqq \beta(\F_t)$ is a trace that generally depends on the history $\F_t$.
We define $\beta_0 \coloneqq 1$ to ensure that the first TD error, $\delta^\pi_0$, is applied.
In \Cref{sect:analysis}, we characterize the values of $\beta_t$ for $t \geq 1$ that lead to convergence.

The major analytical challenge of $\M$---and its main novelty---is the complex dependence on the sequence $\F_t$.
This makes the operator difficult to analyze mathematically, as the terms in the series
$1 + \gamma \beta_1 + \gamma^2 \beta_2 + \cdots$
generally share no common factors that would allow a recursive formula for eligibility traces.
Some off-policy methods, however, cannot be described by factored traces, and therefore removing this assumption is necessary to understand existing algorithms (see \Cref{sect:example_algorithms}), while also paving the way for new credit-assignment methods.
In the special case where $\beta_t$ does factor into Markov coefficients, i.e., $\smash{\beta_t = \prod_{k=1}^t c_k}$,
then \Cref{eq:M_op_def} reduces to \Cref{eq:R_op_def}, taking us back to the per-decision setting studied by \citet{munos2016safe}.
\emph{$\mathcal{M}$, therefore, unifies per-decision and trajectory-aware methods.}

\section{Unifying Off-Policy Algorithms}
\label{sect:example_algorithms}

The operator $\mathcal{M}$ is a strict generalization of the previous operator considered for trace-based methods, allowing us to express existing algorithms in this form.
We provide a non-exhaustive list of examples below with the corresponding $\beta_t$ used in $\mathcal{M}$. 
For brevity, let $\Pi_t \coloneqq \prod_{k=1}^t \rho_k$.

\textbf{Importance Sampling:}
$\beta_t = \lambda^t \Pi_t$
\citep{kahn1951estimation}.
The standard approach for correcting off-policy bias.
Although it is the only unbiased estimator in this list (if $\lambda=1$), it suffers from high variance, making it difficult to utilize.
\looseness=-1

\textbf{$\text{Q}^\pi$($\lambda$):}
$\beta_t = \lambda^t$
\citep{harutyunyan2016q}.
A straightforward algorithm that decays the TD errors by a fixed constant.
The algorithm does not require explicitly knowing $\mu$, which is desirable, but can diverge if $\pi$ and $\mu$ differ too much \citep[][Theorem~1]{harutyunyan2016q}.

\textbf{Tree Backup:}
${\beta_t \!=\! \prod_{k=1}^t \lambda \pi(A_k|S_k)}$
\citep{precup2000eligibility}.
A method that automatically cuts traces according to the product of probabilities under $\pi$, which forms a conservative lower bound on the IS estimate.
Tree Backup converges for any behavior policy $\mu$, but it is not efficient since traces are cut excessively---especially in the on-policy case.

\textbf{Retrace:}
$\beta_t = \prod_{k=1}^t \lambda \min (1,~\rho_k)$
\citep{munos2016safe}.
A convergent algorithm for arbitrary policies $\pi$ and $\mu$ that remains efficient in the on-policy case because it does not cut traces (if $\lambda = 1$);
however, the fact that $\beta_t$ never increases can cause the trace products to decay too quickly in practice \citep{mahmood2017multi,rowland2020adaptive}.

All of the above can be analyzed using a per-decision operator.
The next two, on the other hand, have weightings based on the entire trajectory.
We use the theory for our general $\M$ operator to prove properties about these methods.

\textbf{Recursive Retrace:}
${\beta_t \!=\! \lambda \min (1,\beta_{t-1} \rho_t)}$
\citep{munos2016safe}.
A modification to Retrace conjectured to lead to faster learning. 
It clips large products of ratios, rather than individual ratios.
Its convergence for control is an open question, which we solve in \Cref{sect:analysis}.

\textbf{Truncated Importance Sampling:
$\beta_t = \lambda^t \min (1,~\Pi_t)$}
\citep{ionides2008truncated}.
A simple but effective method to combat the variance of IS.
Variations of this algorithm have been applied in the reinforcement learning
literature \citep[e.g.,][]{uchibe2004competitive, wawrzynski2007truncated, wawrzynski2009real, wang2017sample}, but, to our knowledge, its convergence in an MDP setting has not been studied.
In \Cref{sect:examples_of_convergence}, we show that it can diverge in at least one off-policy problem.

\section{Convergence Analysis}
\label{sect:analysis}

In this section, we study the convergence properties of the $\mathcal{M}$ operator for policy evaluation and control.
It will be convenient to re-express \Cref{eq:M_op_def} in vector notation for our analysis.
To do this, let us first bring the expectation inside the sum, by linearity of expectation:
\begin{align}
    \nonumber
    &(\mathcal{M}Q)(s,a) = Q(s,a)~+ \\*
    \label{eq:M_op_def_linearity}
    &\qquad \sum_{t=0}^\infty \gamma^t \expect{\big}{\mu}{\beta_t \delta^\pi_t}{(S_0,A_0)=(s,a)}
    .
\end{align}
To write \Cref{eq:M_op_def_linearity} in vector form, we define an operator $B_t$ such that, for an arbitrary vector $X$ in $\R^n$,
\begin{equation}
    \label{eq:B_op_def}
    (B_t X)(s,a) \!\coloneqq\! \expect{\big}{\mu}{\beta_t X(S_t,A_t)}{(S_0,A_0)=(s,a)}
    ,
\end{equation}
allowing us to express the $\mathcal{M}$ operator as
\begin{equation}
    \label{eq:M_op_def_vector}
    \mathcal{M}Q = Q + \sum_{t=0}^\infty \gamma^t B_t (T_\pi Q - Q)
    .
\end{equation}
$B_t$ is a linear operator and hence can be represented as a matrix in $\R^{n \times n}$, the elements of which are nonnegative.
Each element of $B_t$, row-indexed by $(s,a)$ and column-indexed by $(s',a')$, has the form
\begin{align}
    \nonumber
    &\!B_t((s,a),(s',a')) \!=\!
    \Pr_{\mu}((S_t,A_t)\!=\!(s',a') \!\mid\! (S_0,A_0)\!=\!(s,a)) \\*
    \label{eq:B_op_def_elements}
    &\qquad \times~\expect{\big}{\mu}{\beta_t \!}{\! (S_0,A_0)\!=\!(s,a), (S_t,A_t)\!=\!(s',a')}
    .
\end{align}
We justify this form in \Cref{app:M_op}.
Note that $B_0 = I$, the identity matrix, because of our earlier definition of $\beta_0 \coloneqq 1$.
In the following sections, all inequalities involving vectors or matrices should be interpreted element wise.
We let $\norm{X} \coloneqq \norm{X}_\infty$ for a matrix (or vector) $X$, which corresponds to the maximum absolute row sum of $X$.
We also define $\ones \in \R^n$ to be the vector of ones, such that $X \ones$ gives the row sums of $X$.

\subsection{Convergence for Policy Evaluation}

We start in the off-policy policy evaluation setting. 
Specifically, our goal is to prove that the repeated application of the $\mathcal{M}$ operator to an arbitrarily initialized vector $Q \in \R^n$ converges to $Q^\pi$.

\begin{condition}
    \label{cond:convergence}
    $\beta_t \leq \beta_{t-1} \rho_t$,
    $\forall~\F_t$,
    $\forall~t \geq 1$.
\end{condition}

\begin{theorem}
    \label{theorem:evaluation}
    If \Cref{cond:convergence} holds, then $\M$ is a contraction mapping with $Q^\pi$ as its unique fixed point.
    Consequently, $\smash{\lim_{i \to \infty} \M^i Q = Q^\pi}$, $\forall~Q \in \R^n$.
\end{theorem}
\proofglue
\begin{proof}
    \input{proofs/theorem_evaluation}
\end{proof}
\proofglue
Given that the ratio $\smash{\frac{\beta_t}{\beta_{t-1}}}$ is bounded by $\rho_t$ (\Cref{cond:convergence}), the $\M$ operator converges to $Q^\pi$.
Intuitively, we can think of this ratio as the \emph{effective} per-decision factor at time $t$;
convergence is guaranteed whenever this factor is no greater than $\rho_t$, analogous to the convergence result for the $\mathcal{R}$ operator \citep[][Theorem~1]{munos2016safe}.
Our theorem implies the existence of a space of convergent trajectory-aware algorithms, because each trace $\beta_t$ can be chosen arbitrarily so long as it always satisfies the bound on this ratio.
\looseness=-1

\subsection{Convergence for Control}

We now consider the more challenging setting of control.
Given sequences of target policies $(\pi_i)_{i \geq 0}$ and behavior policies $(\mu_i)_{i \geq 0}$, we aim to show that the sequence of value functions $(Q_i)_{i \geq 0}$ given by $Q_{i+1} \coloneqq \mathcal{M}_i Q_i$ converges to $Q^*$.
Here, $\mathcal{M}_i$ is the $\mathcal{M}$ operator defined for $\pi_i$ and $\mu_i$.

Compared to the convergence proof of the $\mathcal{R}$ operator \citep[][Theorem~2]{munos2016safe}, the main novelty of our proof is the fact that the traces under $\mathcal{M}$ are not Markov.
Consequently, we require new techniques to establish bounds on ${Q - Q^*}$, since \Cref{eq:M_op_def} is not representable as an infinite geometric series and so the summation does not have a closed-form expression.
We additionally relax two assumptions in the previous work, on initialization of the value function and on increasing greediness of the policy.
We require only that the target policies become \emph{greedy in the limit}.
We say that a sequence of policies is greedy in the limit if
$T_{\pi_i} Q_i \to T Q_i$ as $i \to \infty$.
We discuss the significance of these relaxations to the assumptions in \Cref{sect:discussion}.
\looseness=-1

First, let $C_i \coloneqq \sum_{t=0}^\infty \gamma^t B_t$ for the policies $\pi_i$ and $\mu_i$, and write the $\M$ operator at iteration $i$ as
\begin{equation}
    \label{eq:Mi_op_def}
    \M_i Q = Q + C_i (T_{\pi_i} Q - Q)
    .
\end{equation}
We now present our convergence theorem for control.
\begin{restatable}{theorem}{control}
    \label{theorem:control}
    Consider a sequence of target policies $(\pi_i)_{i \geq 0}$ and a sequence of arbitrary behavior policies $(\mu_i)_{i \geq 0}$.
    Let $Q_0$ be an arbitrary vector in $\R^n$ and define the sequence $Q_{i+1} \coloneqq \M_i Q_i$, where $\M_i$ is the operator defined by \Cref{eq:Mi_op_def}.
    Assume that $(\pi_i)_{i \geq 0}$ is greedy in the limit, and let $\epsilon_i \geq 0$ be the smallest constant such that
    $T_{\pi_i} Q_i \geq T Q_i - \epsilon_i \norm{Q_i} \ones$.
    If \Cref{cond:convergence} holds for all $i$, then
    \begin{equation}
        \label{eq:almost_contraction}
        \norm{\M_i Q_i - Q^*} \leq \gamma \norm{Q_i - Q^*} + \frac{\epsilon_i}{1-\gamma} \norm{Q_i}
        ,
    \end{equation}
    and, consequently, $\smash{\lim\limits_{i \to \infty} Q_i = Q^*}$.
\end{restatable}
\proofglue
\begin{proof}[Proof (sketch; full proof in \Cref{app:theorem_control}).]
    We define matrices $Z_i$ and $Z_i^*$, which correspond to $Z$ in \Cref{eq:M_op_linear_error} for target policies $\pi_i$ and $\pi^*$, respectively, and behavior policy $\mu_i$.
    We then derive the inequalities
    \begin{equation*}
        Z_i^* (Q_i - Q^*) - \epsilon_i \norm{Q_i} C_i \ones
        \leq \M_i Q_i - Q^*
        \leq Z_i (Q_i - Q^*)
        ,
    \end{equation*}
    which together imply \Cref{eq:almost_contraction}.
    Thus, $\M_i$ is nearly a contraction mapping with $Q^*$ as its unique fixed point, excepting the influence of the $O(\norm{Q_i})$ term.
    However, the greedy-in-the-limit target policies guarantee that $\epsilon_i \to 0$.
    Showing that $\norm{Q_i}$ remains finite completes the proof because $\norm{Q_i - Q^*} \to 0$ must follow.
\end{proof}
\proofglue
The convergence criteria for $\beta_t$ (\Cref{cond:convergence}) is the same for both policy evaluation and control.
In fact, the only additional assumption we need for control is the greedy-in-the-limit target policies.
Crucially, the proof allows arbitrary behavior policies and an arbitrary value function initialization $Q_0$, which we further discuss in \Cref{sect:discussion}.

\subsection{Examples of Convergence and Divergence}
\label{sect:examples_of_convergence}

The generality of the $\M$ operator means that it provides convergence guarantees for a number of credit-assignment methods that we did not discuss in \Cref{sect:example_algorithms}.
These include variable or past-dependent $\lambda$-values
\citep[e.g.,][]{watkins1989learning,singh1996reinforcement,yu2018generalized}.
All of these can be represented in a common form and shown to satisfy \Cref{cond:convergence};
convergence for policy evaluation and control for the instantiated trajectory-aware operator follows as a corollary, since \Cref{cond:convergence} is sufficient to apply \Cref{theorem:evaluation,theorem:control}.
\begin{restatable}{proposition}{generaltdlambda}
    \label{prop:general_td_lambda}
    Any traces expressible in the form
    ${\beta_t = \prod_{k=1}^t \lambda(\F_k) \rho_k}$,
    $\lambda(\F_k) \in [0,1]$,
    satisfy \Cref{cond:convergence}.
\end{restatable}
\proofglue
\begin{proof}
    \input{proofs/prop_general_td_lambda.tex}
\end{proof}
\proofglue

% Intentional new paragraph
In \Cref{sect:example_algorithms}, we also discussed two existing trajectory-aware methods whose convergence is unknown.
We show that Recursive Retrace satisfies our required condition.
\begin{restatable}{proposition}{recursiveretrace}
    \label{prop:recursive_retrace.tex}
    Recursive Retrace satisfies \Cref{cond:convergence}.
\end{restatable}
\proofglue
\begin{proof}
    \input{proofs/prop_recursive_retrace.tex}
\end{proof}
\proofglue
Unfortunately, the traces for Truncated IS do not always satisfy the required bound.
\begin{restatable}{proposition}{truncatedis}
    \label{prop:truncated_is}
    Truncated IS may violate \Cref{cond:convergence}.
\end{restatable}
\proofglue
\begin{proof}
    \input{proofs/prop_truncated_is.tex}
\end{proof}
\proofglue
Because \Cref{theorem:evaluation,theorem:control} cannot be applied, the precise conditions under which Truncated IS converges remains an open problem.
We do know \Cref{cond:convergence} is sufficient for convergence, but it is unlikely to be strictly necessary.
This is because our proofs of \Cref{theorem:evaluation,theorem:control} use this assumption to guarantee that the matrix $Z$ in \Cref{eq:M_op_linear_error} has nonnegative elements, making it straightforward to show that its row sums are sufficiently bounded to guarantee that $\M$ is a contraction mapping.
However, $\M$ could remain a contraction mapping even when $Z$ has negative elements, so long as $\norm{Z} < 1$.
It could theoretically be the case for Truncated IS that $Z$ occasionally contains negative elements but $\norm{Z}$ is still bounded enough to permit convergence.

Nevertheless, we are able to find at least one off-policy problem for which this is not true, implying that certain initializations of the value function could ultimately cause Truncated IS to diverge.
\begin{counterexample}[Off-Policy Truncated IS]
    \label{counterexample:truncated_is}
    Consider Truncated IS with $\lambda=1$, so
    $\beta_t = \min(1,~\Pi_t)$.
    Assume the MDP has one state and two actions:
    $\mathcal{S} = \{s\}$ and $\mathcal{A} = \{a_1,a_2\}$, the behavior policy $\mu$ is uniform random, and $\pi$ selects $a_1$ with probability $p \in (0,1)$ and selects $a_2$ otherwise.
    When $p=0.6$ and $\gamma=0.94$, then $\norm{Z} > 1$.
\end{counterexample}
Many choices of $p$ and $\gamma$ make $\norm{Z} > 1$, but we discuss specific ones for the counterexample in \Cref{app:counterexample_truncated_is}.

Compared to the per-decision case, where \citet{munos2016safe} showed that arbitrary trace cuts always produce a convergent algorithm, this result is surprising.
Why would the analogous result---in which we ensure that $\beta_t \leq \Pi_t$ for all timesteps---not hold here?
After all, clipping $\beta_t$ such that it never exceeds the IS estimate $\Pi_t$ would be expected to simply incur bias in the return estimation.
For some insight, assume the following expectations are conditioned on ${(S_0,A_0)=(s,a)}$, and observe that \Cref{eq:M_op_def} is equivalent to
\begin{equation*}
    \Emu{\Bigg}{\sum_{t=0}^\infty \gamma^t \beta_t R_t} + \Emu{\Bigg}{\sum_{t=1}^\infty \gamma^t (\beta_{t-1} \rho_t - \beta_t) Q(S_t,A_t)}
    .
\end{equation*}
We show the derivation in \Cref{app:M_op}.
The first term is a (partially) bias-corrected estimate of the discounted return.
The second term is a weighted combination of value-function bootstraps, whose weights are nonnegative when \Cref{cond:convergence} is met.
If the condition is violated on any timestep, then we may actually be subtracting bootstraps from the return estimate, which does not seem sensible.
We believe this is related to the root cause of divergence in \Cref{counterexample:truncated_is};
however, it remains open whether \Cref{cond:convergence} is necessary or merely sufficient.

As our next counterexample example will demonstrate, this effect can even cause divergence in \emph{on-policy} settings.
\begin{counterexample}[On-Policy Binary Traces]
    \label{counterexample:binary}
    Assume the MDP has one state and two actions: $\mathcal{S} = \{s\}$ and $\mathcal{A} = \{a_1,a_2\}$.
    Define a trajectory-aware method such that $\beta_t=1$ if $A_t=a_1$ and $\beta_t=0$ if $A_t=a_2$ (without loss of generality).
    Assume $\pi$ and $\mu$ are uniform random.
    When $\gamma \geq \frac{2}{3}$, then $\norm{Z} \geq 1$.
\end{counterexample}
We provide details in \Cref{app:counterexample_binary}.
Even though $\beta_t \leq \Pi_t = 1$ always, we are able to produce a non-contraction.
The method either fully cuts a trace or does not cut it at all, producing backups that consist of a sparse sum of on-policy TD errors.
It is therefore surprising that divergence occurs.
For the same reason we described above, the non-Markov nature of the trace appears to sometimes cause adverse bootstrapping effects;
in this instance, the ability to examine each trajectory allows the method to strategically de-emphasize certain state-action pairs, ultimately producing a detrimental effect on learning.
Notice that \Cref{cond:convergence} is indeed violated in this case because there is always some chance that $\beta_t = 1$ after $\beta_{t-1} = 0$.
If we add the restriction that $\beta_{t-1} = 0 \implies \beta_t = 0$, i.e., we permanently cut the traces, then convergence is reestablished by \Cref{theorem:evaluation}.

\subsection{Discussion}
\label{sect:discussion}

In this section, we summarize our main theoretical contributions and their significance.
We focused on characterizing the contraction properties of the $\M$ operator, both for policy evaluation and control, in the tabular setting.
These results parallel those for the $\mathcal{R}$ operator underlying Retrace, where $\M$ is a strict generalization of $\mathcal{R}$.
These results indicate that using fixed-point updates, like dynamic programming and temporal difference learning updates, may have divergence issues.
It does not, however, imply other algorithms, such as gradient-based algorithms, cannot find these fixed points.
We show the fixed points still exist and are unbiased, but that algorithms based on iterating with the $\M$ operator might diverge.
\looseness=-1

\textbf{Removal of the Markov assumption.}
Removing the Markov (per-decision) assumption of the $\mathcal{R}$ operator \citep{munos2016safe} to enable trajectory-aware eligibility traces was our primary goal.
When the trace factors are Markov, the operator $B_t$ is independent of $t$, allowing the sum
$\sum_{t=0}^\infty \gamma^t B_t$
to be reduced to
$\sum_{t=0}^\infty (\gamma P_{c\mu})^t$
for a linear operator
$P_{c\mu}$.
The resulting geometric series can then be evaluated analytically, as was done by \citet{munos2016safe}.
In our proofs, we avoided the Markov assumption by directly analyzing the infinite summation, which generally does not have a closed-form expression.
Our work is the first to do this, establishing the first convergence guarantees for general trajectory-aware methods.

\textbf{Arbitrary initialization of the value function.}
We permit any initialization of $Q_0$ in the control setting.
In contrast, \citet{munos2016safe} made the assumption that
$T_{\pi_0} Q_0 - Q_0 \geq 0$
in order to produce a lower bound on $\mathcal{R}_i Q_i - Q^*$, accomplished in practice by a pessimistic initialization of the value function:
$Q_0(s,a) = - \norm{R} \mathbin{/} (1-\gamma)$,
$\forall~(s,a) \in \mathcal{S} \times \mathcal{A}$.
Since $\mathcal{R}$ is a special case of our operator $\M$ where each trace $\beta_t$ factors into Markov coefficients, we deduce as a corollary that Retrace and all other algorithms described by $\mathcal{R}$ do not require pessimistic initialization for convergence.
\looseness=-1

\textbf{Greedy-in-the-limit policies.}
Our requirement of greedy-in-the-limit target policies in \Cref{theorem:control} is less restrictive than the increasingly greedy policies proposed by \citet{munos2016safe}.
We need only
$\lim_{i \to \infty} T_{\pi_i} Q_i = T Q_i$,
and we do not force the sequence of target policies to satisfy
$P_{\pi_{i+1}} Q_{i+1} \geq P_{\pi_i} Q_{i+1}$.
This implies that the agent may target non-greedy policies for any finite period of time, as long as the policies do eventually become arbitrarily close to the greedy policy.
As a corollary, increasingly greedy policies are not necessary for the optimal convergence of Retrace and other per-decision methods.

\section{Recency-Bounded Importance Sampling}
\label{sect:rbis}

\Cref{theorem:evaluation} guarantees convergence to $Q^\pi$ whenever
\Cref{cond:convergence} holds, but we do not expect that all choices of coefficients that satisfy this condition will perform well in practice.
At one extreme, if
$\smash{\beta_t \leq \prod_{k=1}^t \lambda \min(1,~\rho_k)}$
for every $\mathcal{F}_t$, then we have a method that cuts coefficients more aggressively than Retrace does;
it seems unlikely that such a method would learn faster than Retrace, or other per-decision methods.
At the other extreme, when
$\beta_t = \Pi_t$
for every $\mathcal{F}_t$, we recover the standard IS estimator, which suffers from high variance and is often ineffectual.
We therefore know that it is possible to have a method that preserves traces \emph{too much}, to the point of being detrimental.
Thus, it is important to maintain some minimum efficiency by avoiding unnecessary cuts, yet equally important to control the overall variance of the traces.

Intuitively, we want something that falls between Retrace and IS in terms of trace cutting, in order to quickly backpropagate credit while still managing the variance.
We further hypothesize that effective trajectory-aware methods will first compute $\beta_{t-1} \rho_t$---i.e., the maximum trace permitted by \Cref{cond:convergence}---and then apply some transformation that limits its magnitude to reduce variance.
This ensures that traces are cut only as needed.

We propose one method, Recency-Bounded Importance Sampling (RBIS), which achieves this by cutting the traces only when they exceed an exponentially decaying threshold.
Specifically, we define
\begin{equation}
    \tag{RBIS}
    \beta_t = \min (\lambda^t,~\beta_{t-1} \rho_t)
    .
\end{equation}
It is easy to see that RBIS always converges, by construction.
\begin{restatable}{proposition}{rbis}
    \label{prop:rbis}
    RBIS satisfies \Cref{cond:convergence}.
\end{restatable}
\proofglue
\begin{proof}
    \input{proofs/prop_rbis.tex}
\end{proof}
\proofglue
For further insight, we unroll the recursion to obtain
\begin{align}
    \nonumber
    &\min(\lambda^t \!,~\beta_{t-1} \rho_t) \\
    \nonumber
    &\qquad = \min(\lambda^t,~\min(\lambda^{t-1},~\beta_{t-2} \rho_{t-1}) \rho_t) \\
    \nonumber
    &\qquad \cdots \\
    \label{eq:unroll}
    &\qquad = \min(\lambda^t,~\lambda^{t-1} \rho_t,~\lambda^{t-2} \rho_{t-1} \rho_t,~\dots,~\Pi_t)
    .
\end{align}
RBIS effectively takes the minimum of all past, discounted $n$-step IS estimates.
This reveals another property of RBIS:
its traces are never less than those of Retrace, because
\begin{align*}
    \prod_{k=1}^t \lambda \min(1,~\rho_k)
    &\leq \prod_{k=1}^t \min(\lambda,~\rho_k) \\
    &\leq \lambda^{t-j} \!\!\! \prod_{k=t-j+1}^t \!\!\! \rho_k
    ,~\forall~j \in \{0, 1, \dots, t\}
    .
\end{align*}
Since the inequality is true for all $j$, it is not possible for Retrace's traces to exceed any of the arguments to the $\min$ function in \Cref{eq:unroll}.
We have achieved exactly what we wanted earlier:
a method that falls somewhere between Retrace and IS in regard to trace cutting.
This does not automatically mean that RBIS will outperform Retrace, though, since preserving the magnitude of the trace $\beta_t$ too much can lead to high variance.
However, we do expect RBIS to perform well in decision-making problems in which a few critical actions largely determine the long-term outcome of an episode.
In such scenarios, the agent's bottleneck to learning is its ability to assign meaningful credit to these critical actions over a potentially long time horizon.

In order to test this empirically, we construct an environment called the Bifurcated Gridworld (see \Cref{fig:lambda_sweep}).
This $5 \times 5$ deterministic gridworld has walls arranged such that two unequal-length paths from the start (S) to the goal (G) are available.
The agent may move up, down, left, or right;
taking any of these actions in the goal yields a reward of $+1$ and terminates the episode.
The problem is discounted ($\gamma = 0.9$) to encourage the agent to learn the shorter path.
Importantly, the action taken at the bifurcation (B) solely determines which path the agent follows, and quickly assigning credit to this state is paramount to learning the task.

We compare RBIS against Retrace, Truncated IS, and Recursive Retrace when learning this task from off-policy data.
Both behavior and target policies were $\epsilon$-greedy with respect to the value function $Q$.
The target policy used $\epsilon=0.1$.
The behavior policy used a piecewise schedule:
$\epsilon=1$ for the first 5 episodes and then $\epsilon=0.2$ afterwards.
The agents learned from online TD updates with eligibility traces (see \Cref{app:implementation} for pseudocode).
The policies were updated only at the end of each episode, and then the discounted return obtained by a near-greedy policy ($\epsilon=0.05$) was evaluated.
The area under the curve (AUC) of each resulting learning curve was calculated, with the highest AUC achieved over a grid search of stepsizes being plotted for each $\lambda$-value in \Cref{fig:lambda_sweep}.
We averaged the results over 1,000 independent trials and indicate the 95\% confidence interval by the shaded regions.
In \Cref{app:additional_results}, we repeated the experiment for three more gridworld topologies;
the obtained results are qualitatively similar to \Cref{fig:lambda_sweep}.
Our experiment code is available online.\footnote{\url{https://github.com/brett-daley/trajectory-aware-etraces}}
\looseness=-1

We make several observations regarding the results in \Cref{fig:lambda_sweep}.
First, the peak performance obtained by RBIS is significantly higher than that of the other three methods.
This is notable because both Truncated IS and Recursive Retrace are also trajectory aware, indicating that different implementations of trajectory awareness are beneficial to varying degrees.
In particular, the preservation of long-term eligibilities is not sufficient on its own to guarantee strong performance in general, as it appears that \emph{when} and \emph{how much} the traces are cut are important considerations as well.
The role of $\lambda$ as a decay hyperparameter is evidently critical for all methods to achieve their maximum performance, since $\lambda=1$ never leads to the fastest learning.
In fact, $\lambda \to 1$ is especially catastrophic for Truncated IS, which we believe is related to the divergence issue identified in \Cref{sect:examples_of_convergence}.
Finally, Retrace degrades less for larger $\lambda$, likely because it cuts traces more. 
It would be interesting to develop a trajectory-aware method that obtains the robustness of RBIS but also accounts for larger $\lambda$-values. 

\begin{figure}[t]
    \centering
    \includegraphics[width=\columnwidth]{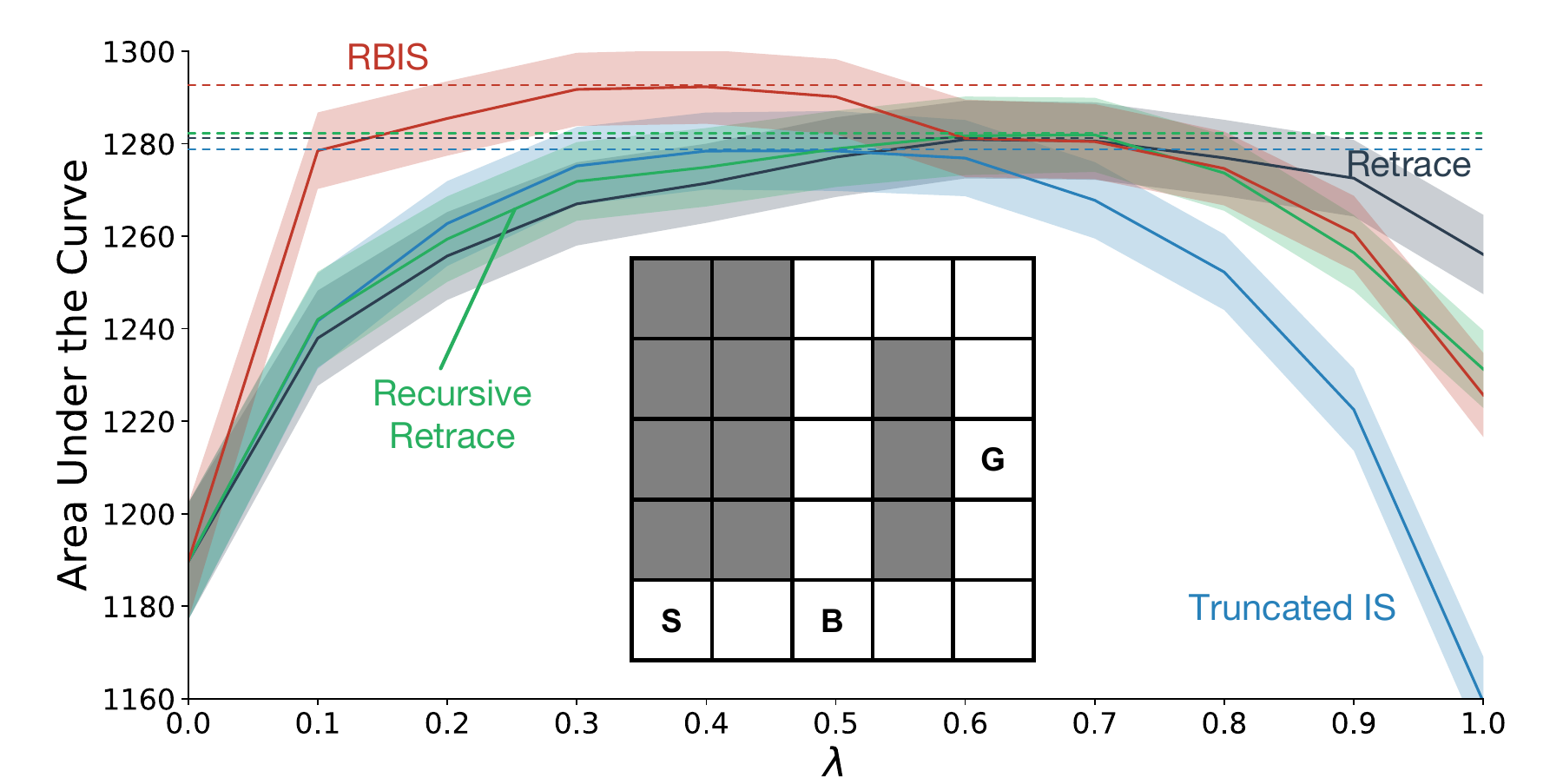}
    \vspace{-0.2in}
    \caption{
        The Bifurcated Gridworld environment.
        The choice made at B greatly impacts the discounted return ultimately earned.
        We plot the AUC obtained by four off-policy methods across the $\lambda$-spectrum.
        The dashed horizontal lines mark the highest AUC achieved by each method.
    }
    \label{fig:lambda_sweep}
    \vspace{-0.1in}
\end{figure}

\section{Conclusion}
\label{sect:conclusion}

In this work, we extended theory for per-decision eligibility traces to trajectory-aware traces.
This extension allows us to consider a broader family of algorithms, with more flexibility in obtaining off-policy corrections.
Specifically, we introduced the $\M$ operator as a generalization of the $\mathcal{R}$ operator, and a sufficient condition to ensure convergence under $\M$.
Using our general result, we established the first convergence guarantee for an existing trajectory-aware method, Recursive Retrace, in the control setting.
We also showed that Truncated IS may violate our condition and provided a counterexample showing that it can diverge.

We also proposed a new trajectory-aware method, RBIS, that demonstrates one instance of how trajectory awareness can be utilized for faster learning in off-policy control tasks.
RBIS is able to outperform the other trajectory-aware methods that we tested in the Bifurcated Gridworld, suggesting that it possesses at least one unique property that is beneficial for long-term, off-policy credit assignment.
It would be interesting to search for additional beneficial properties in future work, in order to better characterize off-policy methods that reliably lead to efficient and stable learning in challenging reinforcement learning environments.

This work focused on convergence \emph{in expectation};
a natural next step is to extend this result to the stochastic algorithms used in practice. 
Previous results for TD learning rely primarily on the properties of the expected update, with additional conditions on the noise in the update and appropriately annealed step sizes \citep[see][Section~4.3]{bertsekas1996neuro}.
Similar analysis should be applicable, given that we know the expected update with $\M$ is a contraction mapping when \Cref{cond:convergence} is met.

An important next step is extending these methods and results to function approximation.
Incorporating these traces into deep reinforcement learning methods that rely on experience replay \citep{lin1992self} should be straightforward.
Multistep returns can be computed offline in the replay memory, and then randomly sampled in minibatches to train the neural network.
Using a TD learning update, though, can suffer from convergence issues under function approximation and off-policy learning;
this has been previously resolved by developing gradient-based updates  \citep{sutton2009fast,touati2018convergent}.
An important next step is to develop a gradient-based trajectory-aware algorithm. 
The contraction properties of the operator still impact the quality of the solution, as has been shown to be the case for other off-policy approaches \citep{patterson2022ageneralized}.
The insights in this work, therefore, may provide insights on how to get quality solutions with gradient-based approaches.

%% file: figures/fig_tightrope.tex
\begin{figure}[t]
    \centering
    \includegraphics[width=0.8\columnwidth]{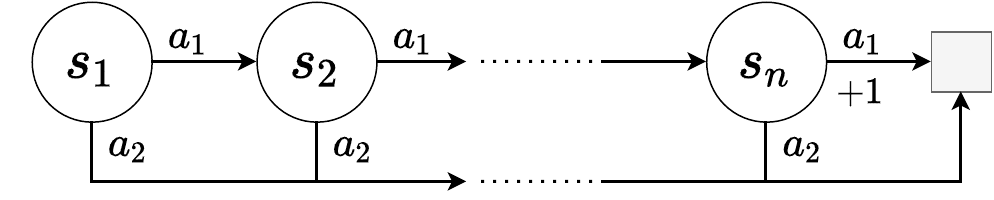}
    \caption{
        The Tightrope Problem.
        Starting from state $s_1$, the agent must take a specific sequence of $n$ actions to receive $+1$ reward.
    }
    \label{fig:tightrope}
\end{figure}

%% file: proofs/theorem_evaluation.tex
In \Cref{lemma:fp_diff} (Appendix~\ref{app:lemma_fp_diff}), we show that $Q^\pi$ is a fixed point of $\mathcal{M}$ and that
\begin{equation}
    \label{eq:M_op_linear_error}
    \mathcal{M}Q - Q^\pi = Z (Q - Q^\pi)
    ,
\end{equation}
where
$Z \coloneqq \sum_{t=1}^\infty \gamma^t (B_{t-1} P_\pi - B_t)$.
In \Cref{lemma:Z_row_sums} (\Cref{app:lemma_Z_row_sums}), we also show that ${Z \geq 0}$ and ${Z \ones \leq \gamma}$ using the assumption that
$\beta_t \leq \beta_{t-1} \rho_t$, $\forall~\F_t$, $\forall~t \geq 1$ (\Cref{cond:convergence}).
Consequently, ${Z (Q - Q^\pi)}$ is a vector whose components each comprise a nonnegative-weighted combination of the components of $Q - Q^\pi$, where the weights add up to at most~$\gamma$.
This means
${\norm{\mathcal{M}Q - Q^\pi} \leq \gamma \norm{Q - Q^\pi}}$,
and $\mathcal{M}$ is a contraction mapping.
Its fixed point, $Q^\pi$, must therefore be unique by the Banach fixed-point theorem, implying that $\smash{\lim_{i \to \infty} \mathcal{M}^i Q = Q^\pi}$ for every $Q \in \R^n$ when $\gamma < 1$.

%% file: proofs/prop_general_td_lambda.tex
$\beta_t = \beta_{t-1} \lambda(\F_t) \rho_t \leq \beta_{t-1} \rho_t$.

%% file: proofs/prop_recursive_retrace.tex
For Recursive Retrace,
$\beta_t = c_t \beta_{t-1}$,
where
$c_t = \lambda \smash{\min\!\left(\frac{1}{\beta_{t-1}},~\rho_t\right)}$ \citep[][Eq.~9]{munos2016safe}.
This means
\begin{align}
    \nonumber
    \beta_t
    &= \lambda \min\!\left(\frac{1}{\beta_{t-1}},~\rho_t \right) \beta_{t-1} \\
    \nonumber
    &= \lambda \min\left(1,~\beta_{t-1} \rho_t \right) \\
    &\leq \beta_{t-1} \rho_t
    ,
\end{align}
which is the bound required by \Cref{cond:convergence}.

%% file: proofs/prop_truncated_is.tex
We show this by providing a counterexample.
Recall that Truncated IS has $\beta_t = \lambda^t \min (1,~\Pi_t)$.
Assume a trajectory $\F_t$ such that $\Pi_{t-1} = 2$ and
$\smash{\frac{1}{2} < \rho_t < \lambda}$.
(It is straightforward to define an MDP, behavior policy, and target policy to create such a trajectory.)
Because $\rho_t > 1 / \Pi_{t-1}$, then $\Pi_t = \Pi_{t-1} \rho_t > 1$.
Thus, $\beta_t = \lambda^t$ and $\beta_{t-1} = \lambda^{t-1}$,
and \Cref{cond:convergence} is violated because
$\frac{\beta_t}{\beta_{t-1}} = \lambda \not\leq \rho_t$.

%% file: proofs/prop_rbis.tex
$\beta_t = \min(\lambda^t,~\beta_{t-1} \rho_t) \leq \beta_{t-1} \rho_t$.

%% file: paper_appendix.tex
\section{$\M$ Operator Details}
\label{app:M_op}

In \Cref{sect:analysis}, we defined a linear operator $B_t$, where
\begin{equation}
    \tag{\ref{eq:B_op_def}}
    (B_t X)(s,a) = \expect{\big}{\mu}{\beta_t X(S_t,A_t)}{(S_0,A_0)=(s,a)}
    ,
\end{equation}
such that the expected-value version of our $\M$ operator,
\begin{align}
    \tag{\ref{eq:M_op_def}}
    (\M Q)(s,a) &= Q(s,a) + \expect{\Bigg}{\mu}{\sum_{t=0}^\infty \gamma^t \beta_t \delta^\pi_t}{(S_0,A_0)=(s,a)} \\
    \tag{\ref{eq:M_op_def_linearity}}
    &= Q(s,a) + \sum_{t=0}^\infty \gamma^t \expect{\big}{\mu}{\beta_t \delta^\pi_t}{(S_0,A_0)=(s,a)},
\end{align}
is element-wise equivalent to the vector version,
\begin{equation}
    \tag{\ref{eq:M_op_def_vector}}
    \mathcal{M}Q = Q + \sum_{t=0}^\infty \gamma^t B_t (T_\pi Q - Q)
    .
\end{equation}
We claimed that each element of $B_t$ must have the form
\begin{align}
    \tag{\ref{eq:B_op_def_elements}}
    B_t((s,a),(s',a')) =
    \Pr_{\mu}((S_t,A_t)=(s',a') \mid (S_0,A_0)=(s,a))
    \times \expect{\big}{\mu}{\beta_t}{(S_0,A_0)=(s,a), (S_t,A_t)=(s',a')}
    ,
\end{align}
with $(s,a)$ as the row index and $(s',a')$ as the column index.
This is because multiplying this matrix $B_t$ with a vector $X$ results in the same operation as the weighted expected value in \Cref{eq:M_op_def_linearity}:
\begin{align}
    \nonumber
    \sum_{s',a'} B_t((s,a),(s',a')) X(s',a')
    &= \expect{\Bigg}{\mu}{
        \expect{\big}{\mu}{\beta_t}{(S_0,A_0)=(s,a),(S_t,A_t)} \cdot X(S_t,A_t)
    }{(S_0,A_0)=(s,a)} \\
    \nonumber
    &= \expect{\bigg}{\mu}{
        \expect{\big}{\mu}{\beta_t X(S_t,A_t)}{(S_0,A_0)=(s,a),(S_t,A_t)}
    }{(S_0,A_0)=(s,a)} \\
    &= \expect{\big}{\mu}{\beta_t X(S_t,A_t)}{(S_0,A_0)=(s,a)}
    .
\end{align}
So, when $X$ is the expected TD error $T_\pi Q - Q$, \Cref{eq:M_op_def_vector} becomes \Cref{eq:M_op_def_linearity} exactly.

$\M$ is a contraction mapping whenever $\beta_t \leq \beta_{t-1} \rho_t$ for all $t$ (\Cref{cond:convergence}), which \Cref{theorem:evaluation} establishes.
As we discussed in \Cref{sect:examples_of_convergence}, violating this condition can sometimes cause $\M$ to no longer contract, even with on-policy updates.
We can see one plausible reason for this by refactoring the definition of $\M$.
Let $q_t \coloneqq Q(S_t,A_t)$ and
$v_t \coloneqq \sum_{a' \in \mathcal{A}} \pi(a'|S_t) Q(S_t,a')$,
so $\delta^\pi_t = R_t + \gamma v_{t+1} - q_t$.
Further, assume the following expectations are conditioned on $(S_0,A_0)=(s,a)$.
\Cref{eq:M_op_def} is equivalent to
\begin{align}
    \nonumber
    (\M Q)(s,a)
    &= q_0 + \Emu{\Bigg}{\sum_{t=0}^\infty \gamma^t \beta_t (R_t + \gamma v_{t+1} - q_t)} \\
    \nonumber
    &= q_0 + \Emu{\Bigg}{\sum_{t=0}^\infty \gamma^t \beta_t R_t + \sum_{t=1}^\infty \gamma^t \beta_{t-1} v_t - \sum_{t=0}^\infty \gamma^t \beta_t q_t} \\
    \nonumber
    &= \Emu{\Bigg}{\sum_{t=0}^\infty \gamma^t \beta_t R_t + \sum_{t=1}^\infty \gamma^t \beta_{t-1} v_t - \sum_{t=1}^\infty \gamma^t \beta_t q_t} \\
    \nonumber
    &= \Emu{\Bigg}{\sum_{t=0}^\infty \gamma^t \beta_t R_t} + \Emu{\Bigg}{\sum_{t=1}^\infty \gamma^t (\beta_{t-1} v_t - \beta_t q_t)} \\
    &= \Emu{\Bigg}{\sum_{t=0}^\infty \gamma^t \beta_t R_t} + \Emu{\Bigg}{\sum_{t=1}^\infty \gamma^t (\beta_{t-1} \rho_t - \beta_t) q_t}
    ,
\end{align}
and we discussed in \Cref{sect:examples_of_convergence} that these two terms represent a biased return estimate and an infinite sum of weighted value-function bootstraps, respectively.
In particular, this can be problematic if $\beta_t > \beta_{t-1} \rho_t$ because the corresponding bootstrap's weight becomes negative, causing it to get subtracted from the return estimate.

\section{Additional Proofs}

\subsection{Proof of Lemma~\ref{lemma:fp_diff}}
\label{app:lemma_fp_diff}
\begin{lemma}
    \label{lemma:fp_diff}
    $Q^\pi$ is a fixed point of $\mathcal{M}$;
    the difference between $\mathcal{M}Q$ and $Q^\pi$ is given by
    \begin{equation}
        \mathcal{M}Q - Q^\pi = Z (Q - Q^\pi)
        ,
    \end{equation}
    where $Z \coloneqq \sum_{t=1}^\infty \gamma^t (B_{t-1} P_\pi - B_t)$.
\end{lemma}
\begin{proof}
    \input{proofs/lemma_fp_diff}
\end{proof}

\subsection{Proof of Lemma~\ref{lemma:Z_row_sums}}
\label{app:lemma_Z_row_sums}
\begin{lemma}
    \label{lemma:Z_row_sums}
    If \Cref{cond:convergence} holds, then $Z$ has nonnegative elements and its row sums obey $Z \ones \leq \gamma$.
\end{lemma}
\begin{proof}
    \input{proofs/lemma_Z_row_sums}
\end{proof}

\subsection{Proof of Theorem~\ref{theorem:control}}
\label{app:theorem_control}
\control*
\begin{proof}
    \input{proofs/theorem_control}
\end{proof}

\section{Examples of Divergence}

\subsection{Counterexample~\ref{counterexample:truncated_is}: Off-Policy Truncated IS}
\label{app:counterexample_truncated_is}

Our definitions of $\pi$ and $\mu$ give us
\begin{align}
    P_\pi = \begin{bmatrix}
        p & 1-p\\
        p & 1-p
    \end{bmatrix}
    ,&&
    P_\mu = \frac{1}{2} \begin{bmatrix}
        1 & 1\\
        1 & 1
    \end{bmatrix}
    .
\end{align}
Recall that we assumed $\lambda = 1$.
We define the following constant, using the definition of $\beta_t$ for Truncated IS:
\begin{align}
    \nonumber
    \beta^{(1)}_t
    &\coloneqq \expect{\big}{}{\beta_t}{(S_t,A_t)=(s,a_1)} \\
    \nonumber
    &= \sum_{\F_t} \Pr\nolimits_\mu(\F_t\mid(S_t,A_t)=(s,a_1)) \cdot \min\left(1,~\frac{\Pr_\pi(\F_t)}{\Pr_\mu(\F_t)}\right)\\
    \label{eq:memoryless}
    &= \sum_{\F_{t-1}} \Pr\nolimits_\mu(\F_{t-1}) \min\left(1,~\frac{\Pr_\pi(\F_{t-1}) \cdot p}{\Pr_\mu(\F_{t-1}) \cdot \frac{1}{2}}\right) \\
    \nonumber
    &= \sum_{\F_{t-1}} \min\left(\Pr\nolimits_\mu(\F_{t-1}),~2p \cdot \Pr\nolimits_\pi(\F_{t-1})\right) \\
    &= \sum_{\F_{t-1}} \min\left(\frac{1}{2^{t-1}},~2p \cdot \Pr\nolimits_\pi(\F_{t-1})\right)
    .
\end{align}
\Cref{eq:memoryless} is justified because the conditional probability of a trajectory ending in action $a_1$ is just the probability of $\F_{t-1}$ under $\mu$, due to the 1-state (memoryless) MDP.
We can simplify $\smash{\beta^{(1)}_t}$ further by using the binomial theorem to calculate $\Pr_\pi(\F_{t-1}) = p^k (1-p)^{t-1-k}$, where $k \in [0,t-1]$ is the number of times $a_1$ is taken in $\F_{t-1}$.
There are $\binom{t-1}{k}$ trajectories with this same probability.
Therefore,
\begin{align}
    \beta^{(1)}_t
    = \sum_{\F_{t-1}} \min\left(\frac{1}{2^{t-1}},~2p \cdot \Pr\nolimits_\pi(\F_{t-1})\right)
    &= \sum_{k=0}^{t-1} \binom{t-1}{k} \min\left(\frac{1}{2^{t-1}},~2p \cdot p^k (1-p)^{t-1-k} \right)
    .
\end{align}
Likewise, we can compute $\beta^{(2)}_t$ by swapping $p$ and $1-p$ above.
Let $\odot$ denote element-wise multiplication.
Using the fact that $P_\mu^t = P_\mu$, $\forall~t \geq 1$, it follows that
\begin{equation}
    B_t
    = P_\mu^t \odot \begin{bmatrix}
        \beta^{(1)}_t & \beta^{(2)}_t\\
        \beta^{(1)}_t & \beta^{(2)}_t
    \end{bmatrix}
    = \frac{1}{2} \begin{bmatrix}
        \beta^{(1)}_t & \beta^{(2)}_t\\
        \beta^{(1)}_t & \beta^{(2)}_t
    \end{bmatrix}
    .
\end{equation}
Using a computer program to calculate $Z$, assuming that $p=0.6$ and $\gamma=0.94$, we obtain
\begin{equation}
    Z
    = \sum_{t=1}^\infty \gamma^t (B_{t-1} P_\pi - B_t)
    \approx \begin{bmatrix}
        0.704 & -0.436\\
        0.704 & -0.436
    \end{bmatrix}
    .
\end{equation}
Therefore, $\norm{Z} \approx 1.14$, which is not a contraction, and the norm continues to increase for $p > 0.6$ or $\gamma>0.94$.

\subsection{Counterexample~\ref{counterexample:binary}: On-Policy Binary Traces}
\label{app:counterexample_binary}

The policy $\pi$ is uniform random, so we have
\begin{align}
    P_\pi = \frac{1}{2} \begin{bmatrix}
        1 & 1\\
        1 & 1
    \end{bmatrix}
    .
\end{align}
Let $\odot$ denote element-wise multiplication.
Because $\beta_t=1$ only when the trajectory $\F_t$ terminates in $(s,a_1)$ and $\beta_t=0$ otherwise, and since $P_\pi^t = P_\pi$, $\forall~t \geq 1$, we also have
\begin{align}
    B_t
    = P_\pi^t \odot \begin{bmatrix}
        1 & 0\\
        1 & 0
    \end{bmatrix}
    = \frac{1}{2} \begin{bmatrix}
        1 & 0\\
        1 & 0
    \end{bmatrix}
    .
\end{align}
Using a computer program to calculate $Z$, assuming that $\gamma=\frac{2}{3}$, we obtain
\begin{equation}
    Z
    = \sum_{t=1}^\infty \gamma^t (B_{t-1} P_\pi - B_t)
    = \frac{1}{3} \begin{bmatrix}
        -1 & 2\\
        -1 & 2
    \end{bmatrix}
    .
\end{equation}
Therefore, $\norm{Z}=1$, which is not a contraction, and the norm continues to increase for $\gamma > \frac{2}{3}$.

\section{Implementation of Trajectory-Aware Eligibility Traces}
\label{app:implementation}

The implementation of trajectory-aware methods is closely related to that of backward-view TD($\lambda$) in the tabular setting \citep[see, e.g.,][Chapter~7.3]{sutton1998reinforcement}.
On each timestep, an environment interaction is conducted according to the behavior policy $\mu$.
Then, the eligibilities for previously visited state-action pairs are modified, the eligibility for the current state-action pair is incremented, and the current TD error is applied to all state-action pairs in proportion to their eligibilities.
The only difference in the trajectory-aware case is that the eligibilities are not modified by simply multiplying a constant decay factor $\gamma \lambda$.

Arbitrary, trajectory-dependent traces $\beta(\mathcal{F}_t)$, as studied in our theoretical results, can be complicated to implement.
This stems from the fact that the timestep $t$ in the $\mathcal{M}$ operator is defined \emph{relative} to when the updated state-action pair was taken.
In other words, each state-action pair $(S_k,A_k)$ ``disagrees'' on the start of the current trajectory, generating its update from the unique sub-trajectory $(S_k,A_k), \dots, (S_t,A_t)$.
Implementing coefficients of this form would be possible using the general update
\begin{equation}
    Q(S_k,A_k) \gets Q(S_k,A_k) + \alpha \gamma^{t-k} \beta((S_k,A_k), \dots, (S_t,A_t)) \delta^\pi_t
    ,
\end{equation}
where $\alpha \in (0,1]$ is the stepsize, but this would require repeatedly slicing the list of visited state-action pairs $(S_0,A_0), \dots, (S_t,A_t)$.
While this is certainly feasible, it does not easily accommodate vectorization or parallelization.

Fortunately, this level of generality is rarely needed in practice, and specific optimizations can be made depending on the functional form of $\beta$.
For example, Truncated IS defines $\beta$ to be a pure function of the IS estimate $\Pi_t$, which is useful because per-decision eligibility traces can be used to efficiently generate the IS estimates for every state-action pair visited during the episode.
We demonstrate how this can be done in pseudocode (see \Cref{algo:truncated_is}).

Recursive methods like Recursive Retrace and RBIS, where $\beta_t$ explicitly depends on $\beta_{t-1}$, require only two minor changes compared to \Cref{algo:truncated_is} for their implementations.
These changes, which we highlight in \textcolor{blue}{blue} for RBIS in \Cref{algo:rbis}, correspond to the fact that the dynamic array $Y$ is now used to store the previous trace $\beta_{t-1}$ rather than the previous IS estimate $\Pi_{t-1}$ at each timestep.
The computational requirements for the methods remain nearly identical.
The implementation for Recursive Retrace easily follows by changing line~10 of \Cref{algo:rbis} to
\begin{equation}
    Y(k) \gets \lambda \min(1,~Y(k) \cdot \rho_t)
    .
\end{equation}

\section{Additional Experiment Details and Results}
\label{app:additional_results}

We conducted a grid search to find the best stepsize $\alpha$ for every $\lambda$-value for the four off-policy methods we evaluated in the Bifurcated Gridworld (\Cref{sect:rbis}).
Using a training set of 1,000 trials, we searched over $\lambda \in \{0, 0.1, \dots, 1\}$ and $\alpha \in \{0.1,0.3,0.5,0.7,0.9\}$, for a total of 55 hyperparameter combinations.
At the start of each trial, the initial value function $Q$ was sampled from a zero-mean Gaussian distribution with standard deviation $\sigma=0.01$.
We trained each agent for 3,000 timesteps, allowing extra time to complete the final episode.
We then generated learning curves by plotting the 100-episode moving average of these returns as a function of the number of timesteps and calculated their AUCs.
In \Cref{table:hp_search}, we report the stepsize $\alpha$ that led to the highest average AUC for each $\lambda$-value.
Then, using a separate test set of 1,000 trials to avoid bias in the search results, these $\alpha$-values were used to generate the learning curves in \Cref{fig:learning_curves}.
The AUCs for these learning curves were finally used in the creation of the $\lambda$-sweep plot (\Cref{fig:lambda_sweep}).

In \Cref{fig:additional_envs}, we repeated this procedure for three additional, more complex gridworld topologies.
Like the Bifurcated Gridworld, these environments feature one or more bifurcations that make fast credit assignment imperative, as well as additional challenges such as multiple goal cells.
As before, the agent starts in S and receives $+1$ reward for taking any action in a goal, terminating the episode.
The results are qualitatively similar to \Cref{fig:lambda_sweep};
RBIS outperforms the other three methods by a significant margin over the left-hand portion of the $\lambda$-spectrum, and performs similarly to Retrace as $\lambda \to 1$.

\input{figures/table_hp_search.tex}

\clearpage
\input{figures/fig_learning_curves.tex}

\clearpage
\input{figures/fig_additional_envs}

\clearpage
\input{figures/algo_truncated_is}
\input{figures/algo_rbis}

%% file: proofs/lemma_fp_diff.tex
It is evident from \Cref{eq:M_op_def_vector} that $Q^\pi$ is a fixed point of $\mathcal{M}$ because $T_\pi Q^\pi - Q^\pi = 0$, and so $\mathcal{M}Q^\pi = Q^\pi$.
Therefore,
\begin{align*}
    \allowdisplaybreaks
    \mathcal{M}Q - Q^\pi
    &= \mathcal{M}Q - \mathcal{M}Q^\pi \\
    &= Q + \sum_{t=0}^\infty \gamma^t B_t (T_\pi Q - Q) - Q^\pi - \sum_{t=0}^\infty \gamma^t B_t (T_\pi Q^\pi - Q^\pi) \\
    &= Q - Q^\pi + \sum_{t=0}^\infty \gamma^t B_t (T_\pi Q - T_\pi Q^\pi) - \sum_{t=0}^\infty \gamma^t B_t (Q - Q^\pi) \\
    &= \sum_{t=0}^\infty \gamma^t B_t (T_\pi Q - T_\pi Q^\pi) - \sum_{t=1}^\infty \gamma^t B_t (Q - Q^\pi) \\
    &= \sum_{t=0}^\infty \gamma^{t+1} B_t P_\pi (Q - Q^\pi) - \sum_{t=1}^\infty \gamma^t B_t (Q - Q^\pi) \\
    &= \left( \sum_{t=0}^\infty \gamma^{t+1} B_t P_\pi - \sum_{t=1}^\infty \gamma^t B_t \right) (Q - Q^\pi) \\
    &= \left( \sum_{t=1}^\infty \gamma^t (B_{t-1} P_\pi - B_t) \right) (Q - Q^\pi) \\
    &= Z (Q - Q^\pi)
    ,
\end{align*}
which is the desired result.

%% file: proofs/lemma_Z_row_sums.tex
Define the linear operator
$D_t \coloneqq B_{t-1} P_\pi - B_t$
and notice that
${Z = \sum_{t=1}^\infty \gamma^t D_t}$.
We will show that $D_t$ comprises only nonnegative elements, and therefore so does $Z$.
For any $X \in \R^n$, observe that
\begin{align}
    \nonumber
    (D_t X)(s,a)
    &= \expect{\Bigg}{\mu}{\beta_{t-1} \sum_{S_t \in \mathcal{S}} \sum_{A_t \in \mathcal{A}} P(S_t|\F_{t-1}) \pi(A_t|S_t) X(S_t,A_t)}{(S_0,A_0)=(s,a)} \\
    \nonumber
    &\qquad - \expect{\big}{\mu}{\beta_t X(S_t,A_t)}{(S_0,A_0)=(s,a)} \\
    \nonumber
    &= \expect{\Bigg}{\mu}{\beta_{t-1} \sum_{S_t \in \mathcal{S}}
    \sum_{A_t \in \mathcal{A}} P(S_t|\F_{t-1}) \pi(A_t|S_t) X(S_t,A_t)}{(S_0,A_0)=(s,a)} \\
    \nonumber
    &\qquad - \expect{\Bigg}{\mu}{\sum_{S_t \in \mathcal{S}} \sum_{A_t \in \mathcal{A}} P(S_t|\F_{t-1}) \mu(A_t|S_t) \beta_t X(S_t,A_t)}{(S_0,A_0)=(s,a)} \\
    &= \expect{\Bigg}{\mu}{\sum_{S_t \in \mathcal{S}} P(S_t|\mathcal{F}_{t-1}) \sum_{A_t \in \mathcal{A}} \big( \pi(A_t|S_t) \beta_{t-1} - \mu(A_t|S_t) \beta_t \big) X(S_t,A_t)}{(S_0,A_0)=(s,a)}
    .
\end{align}
Since we assumed that
$\beta_t \leq \beta_{t-1} \rho_t$ in \Cref{cond:convergence},
we have
$\pi(A_t|S_t) \beta_{t-1} - \mu(A_t|S_t) \beta_t \geq 0$,
which implies that $D_t \geq 0$.
Furthermore, this holds for all $t \geq 1$, so $Z \geq 0$ follows immediately.

To complete the proof, we show that the row sums of $Z$ are bounded by $\gamma$.
Recall that $P_\pi \ones = \ones$.
Hence,
\begin{align}
    \nonumber
    Z \ones
    &= \sum_{t=1}^\infty \gamma^t (B_{t-1} P_\pi - B_t) \ones \\
    \nonumber
    &= \sum_{t=1}^\infty \gamma^t (B_{t-1} \ones - B_t \ones) \\
    \nonumber
    &= \sum_{t=0}^\infty \gamma^{t+1} B_t \ones - \sum_{t=1}^\infty \gamma^t B_t \ones \\
    \nonumber
    &= \gamma \ones + \sum_{t=1}^\infty \gamma^{t+1} B_t \ones - \sum_{t=1}^\infty \gamma^t B_t \ones \\
    \nonumber
    &= \gamma \ones - (1-\gamma) \sum_{t=1}^\infty \gamma^t B_t \ones \\
    &\leq \gamma \ones
    ,
\end{align}
because $B_t \geq 0$, $\forall~t \geq 1$.

%% file: proofs/theorem_control.tex
We first derive the following upper bound:
\begin{equation}
    T_{\pi_i} Q_i - T Q^*
    = \gamma P_{\pi_i} Q_i - \gamma \max\limits_\pi P_\pi Q^*
    \leq \gamma P_{\pi_i} (Q_i - Q^*)
    .
\end{equation}
From \Cref{eq:Mi_op_def} and because $C_i$ has nonnegative entries, we can deduce that
\begin{align}
    \allowdisplaybreaks
    \label{eq:start_here}
    \M_i Q_i - Q^*
    &= (I - C_i)(Q_i - Q^*) + C_i (T_{\pi_i} Q_i - Q^*) \\
    \nonumber
    &= (I - C_i)(Q_i - Q^*) + C_i (T_{\pi_i} Q_i - T Q^*) \\
    \nonumber
    &\leq (I - C_i)(Q_i - Q^*) + \gamma C_i P_{\pi_i} (Q_i - Q^*) \\
    \label{eq:upper_bound}
    &= Z_i (Q_i - Q^*)
    ,
\end{align}
where $Z_i \coloneqq I - C_i (I - \gamma P_{\pi_i})$.
Notice that $Z_i$ is analogous to the matrix $Z$ in \Cref{eq:M_op_linear_error} because, for policies $\pi_i$ and $\mu_i$,
\begin{align}
    \nonumber
    I - C_i (I - \gamma P_{\pi_i})
    &= I + \sum_{t=0}^\infty \gamma^t B_t (\gamma P_{\pi_i} - I) \\
    \nonumber
    &= I + \sum_{t=0}^\infty \gamma^{t+1} B_t P_{\pi_i} - \sum_{t=0}^\infty \gamma^t B_t \\
    \nonumber
    &= \sum_{t=1}^\infty \gamma^t B_{t-1} P_{\pi_i} - \sum_{t=1}^\infty \gamma^t B_t \\
    &= \sum_{t=1}^\infty \gamma^t (B_{t-1} P_{\pi_i} - B_t)
    .
\end{align}
Next, we derive the following lower bound:
\begin{equation}
    T Q_i - T Q^*
    \geq T_{\pi^*} Q_i - T Q^*
    = \gamma P_{\pi^*} (Q_i - Q^*)
    .
\end{equation}
Additionally, for each policy $\pi_i$, there exists some $\epsilon_i \geq 0$ such that
$T_{\pi_i} Q_i \geq T Q_i - \epsilon_i \norm{Q_i} \ones$
(recall that we defined $\epsilon_i$ to be as small as possible).
Starting again from \Cref{eq:start_here},
and noting that the elements of $C_i$ are nonnegative, we obtain
\begin{align}
    \allowdisplaybreaks
    \nonumber
    \M_i Q_i - Q^*
    &\geq (I - C_i)(Q_i - Q^*) + C_i (T Q_i - Q^*) - \epsilon_i \norm{Q_i} C_i \ones \\
    \nonumber
    &= (I - C_i)(Q_i - Q^*) + C_i (T Q_i - T Q^*) - \epsilon_i \norm{Q_i} C_i \ones \\
    \nonumber
    &\geq (I - C_i)(Q_i - Q^*) + \gamma C_i P_{\pi^*} (Q_i - Q^*) - \epsilon_i \norm{Q_i} C_i \ones \\
    \label{eq:lower_bound}
    &= Z_i^* (Q_i - Q^*) - \epsilon_i \norm{Q_i} C_i \ones
    ,
\end{align}
where we have defined $Z_i^* \coloneqq I - C_i (I - \gamma P_{\pi^*})$.
By \Cref{lemma:Z_row_sums}, since we assumed \Cref{cond:convergence} holds, both $Z_i$ and $Z_i^*$ have nonnegative elements and their row sums are bounded by~$\gamma$.
Therefore, when
$\M_i Q_i - Q^* \geq 0$,
\Cref{eq:upper_bound} implies
\begin{equation}
    \label{eq:pos_bound}
    \norm{\M_i Q_i - Q^*}
    \leq \gamma \norm{Q_i - Q^*}
    ,
\end{equation}
because element-wise inequality for nonnegative matrices implies the inequality holds also for their norms.
When
${\M_i Q_i - Q^* \leq 0}$,
we must use \Cref{eq:lower_bound} and multiply both sides by $-1$ to get nonnegative matrices, giving
\begin{align}
    \nonumber
    \norm{\M_i Q_i - Q^*}
    &\leq \gamma \norm{Q_i - Q^*} + \epsilon_i \norm{Q_i} \norm{C_i} \\
    \label{eq:neg_bound}
    &\leq \gamma \norm{Q_i - Q^*} + \frac{\epsilon_i}{1-\gamma} \norm{Q_i}
    ,
\end{align}
because
$\norm{C_i} \leq \sum_{t=0}^\infty \gamma^t \norm{P_{\pi_i}}^t = (1-\gamma)^{-1}$.
Since \Cref{eq:neg_bound} is looser than \Cref{eq:pos_bound}, its bound holds in the worst case.
It remains to show that this bound implies convergence to $Q^*$.
Observe that
\begin{align}
    \nonumber
    \gamma \norm{Q_i - Q^*} + \frac{\epsilon_i}{1-\gamma} \norm{Q_i}
    &\leq \gamma \norm{Q_i - Q^*} + \frac{\epsilon_i}{1-\gamma} (\norm{Q_i - Q^*} + \norm{Q^*}) \\
    &= \left( \gamma + \frac{\epsilon_i}{1-\gamma} \right) \norm{Q_i - Q^*} + \frac{\epsilon_i}{1-\gamma} \norm{Q^*}
    .
\end{align}
Our assumption of greedy-in-the-limit policies tells us that $\epsilon_i \to 0$ as $i \to \infty$;
thus, there must exist some iteration $i^*$ such that
$\epsilon_i \leq \frac{1}{2} (1-\gamma)^2$, $\forall~i \geq i^*$.
Therefore, for $i \geq i^*$,
\begin{equation}
    \norm{\M_i Q_i - Q^*}
    \leq \frac{1 + \gamma}{2} \norm{Q_i - Q^*} + \frac{\epsilon_i}{1-\gamma} \norm{Q^*}
    .
\end{equation}
If $\gamma < 1$,
then
$\frac{1}{2} (1 + \gamma) < 1$, and since $\norm{Q^*}$ is finite, we conclude that $\norm{Q_i - Q^*} \to 0$ as $i \to \infty$.

%% file: figures/table_hp_search.tex
\begin{table}[h]
    \centering
    \caption{
        The best stepsizes found by our grid search in the Bifurcated Gridworld.
    }
    \vspace{0.1in}
    \begin{tabular}{rccccccccccc}
        \toprule
        $\lambda$ & 0 & 0.1 & 0.2 & 0.3 & 0.4 & 0.5 & 0.6 & 0.7 & 0.8 & 0.9 & 1  \\
        \midrule
        Retrace            & 0.9 & 0.9 & 0.9 & 0.9 & 0.9 & 0.9 & 0.9 & 0.9 & 0.7 & 0.7 & 0.5 \\
        Truncated IS       & 0.9 & 0.9 & 0.9 & 0.9 & 0.9 & 0.9 & 0.7 & 0.5 & 0.5 & 0.5 & 0.3 \\
        Recursive Retrace  & 0.9 & 0.9 & 0.9 & 0.9 & 0.9 & 0.9 & 0.9 & 0.9 & 0.7 & 0.5 & 0.5 \\
        RBIS               & 0.9 & 0.9 & 0.9 & 0.9 & 0.9 & 0.7 & 0.7 & 0.7 & 0.7 & 0.7 & 0.5 \\
        \bottomrule
    \end{tabular}
    \label{table:hp_search}
\end{table}

%% file: figures/fig_learning_curves.tex
\begin{figure}[ht]
    \centering
    \includegraphics[width=0.3\textwidth]{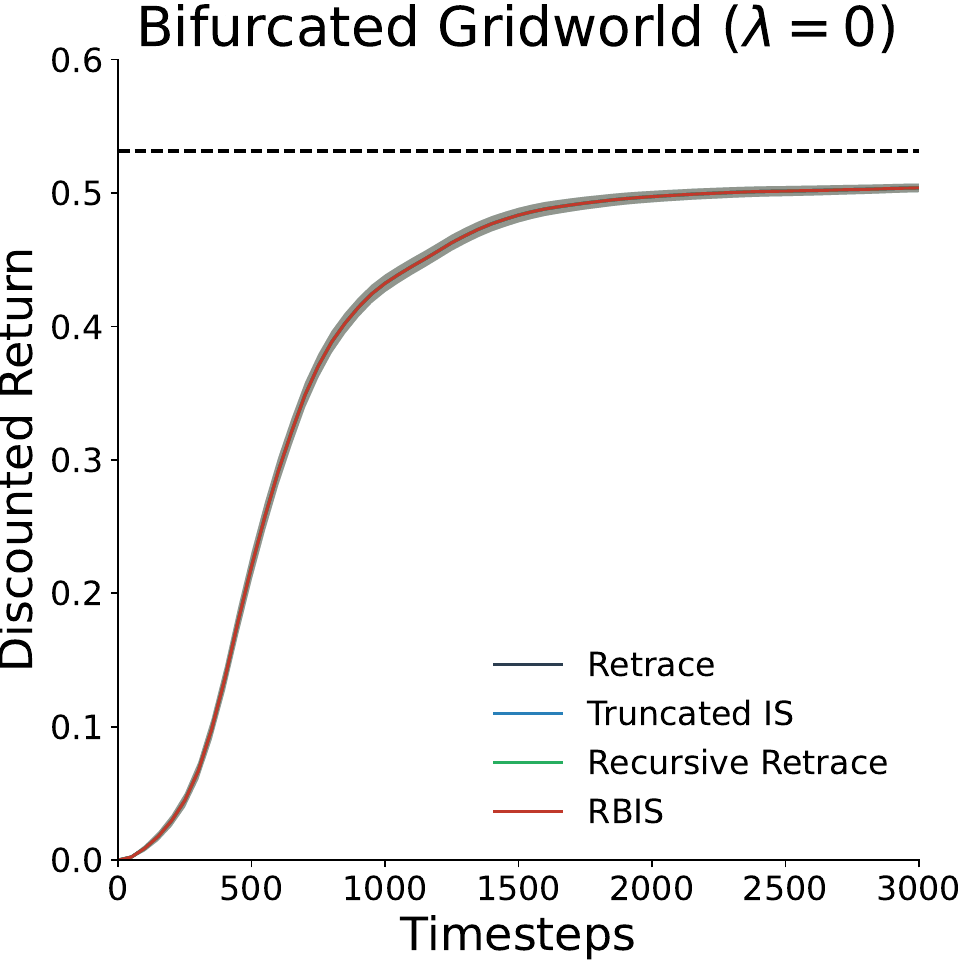}
    \hfill
    \includegraphics[width=0.3\textwidth]{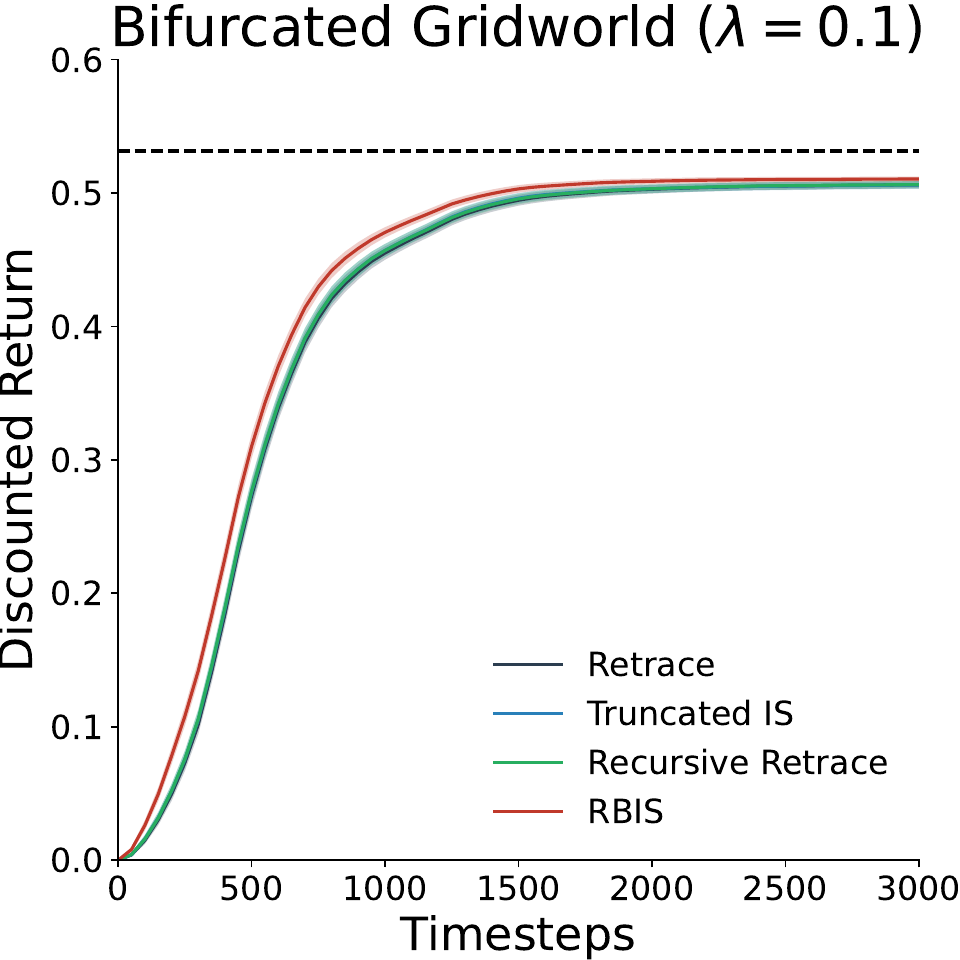}
    \hfill
    \includegraphics[width=0.3\textwidth]{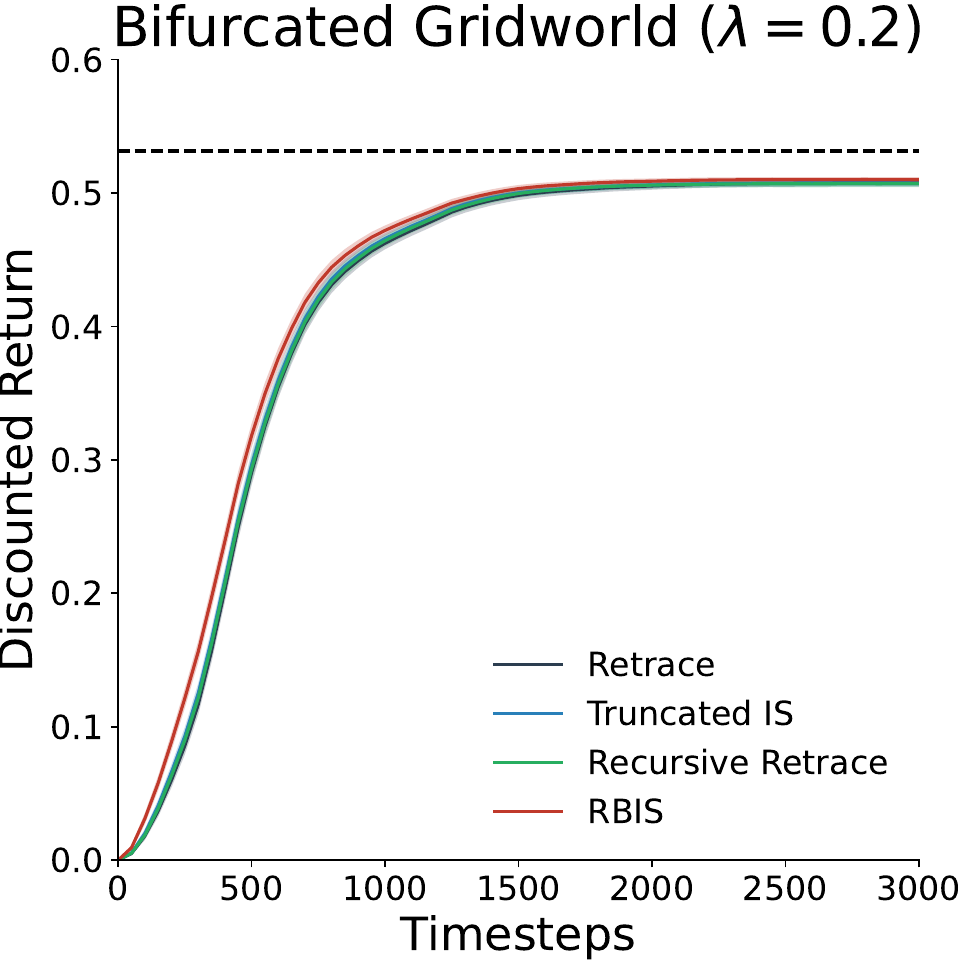}

    \vspace{0.2in}

    \includegraphics[width=0.3\textwidth]{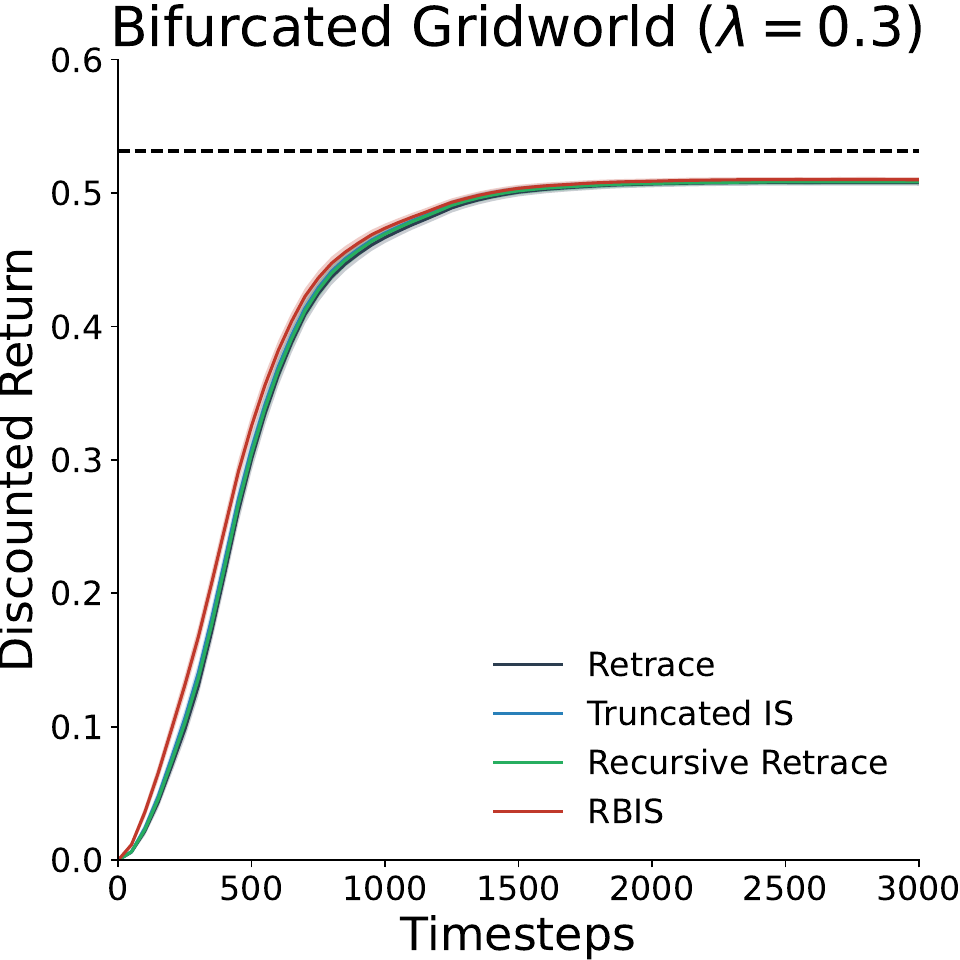}
    \hfill
    \includegraphics[width=0.3\textwidth]{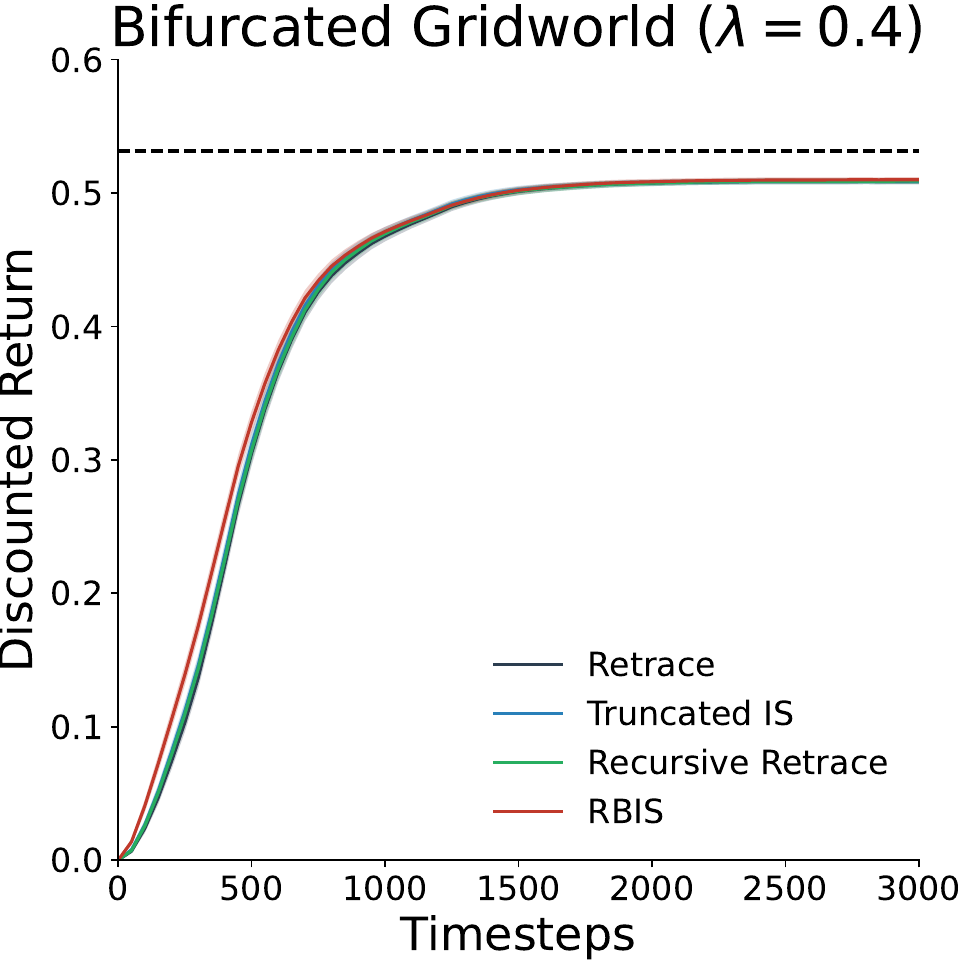}
    \hfill
    \includegraphics[width=0.3\textwidth]{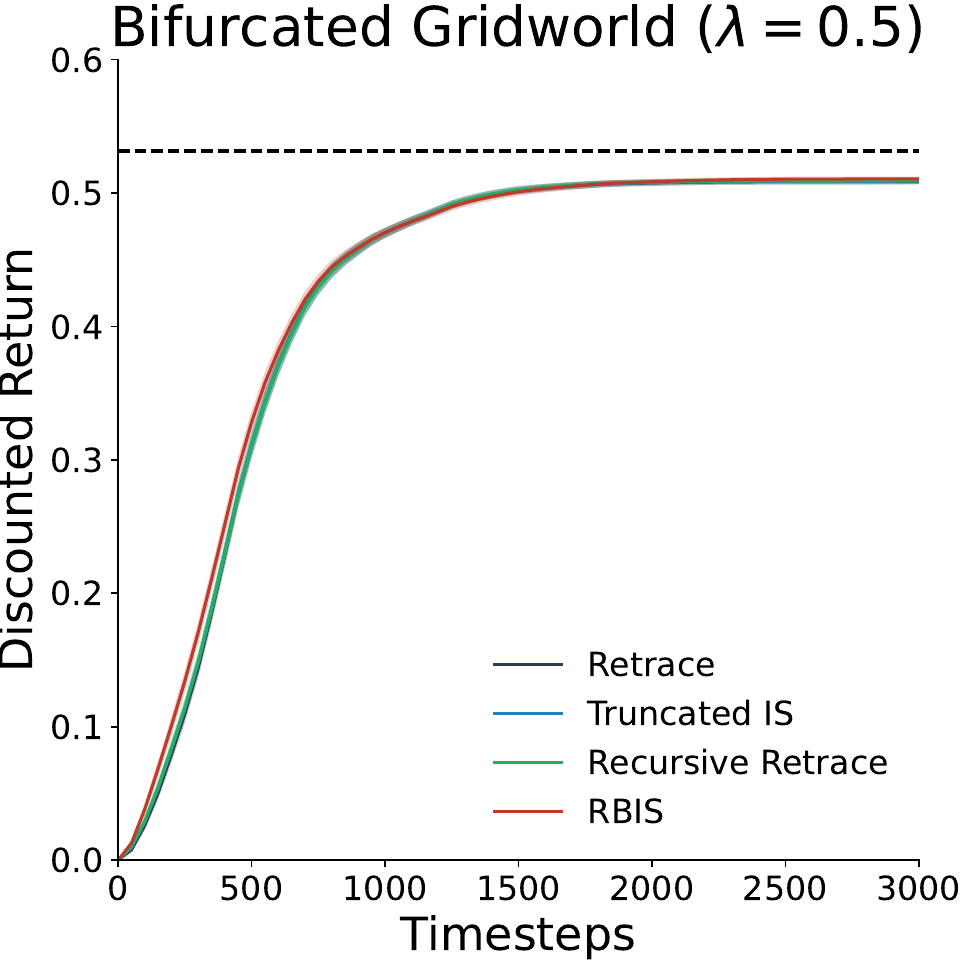}

    \vspace{0.2in}

    \includegraphics[width=0.3\textwidth]{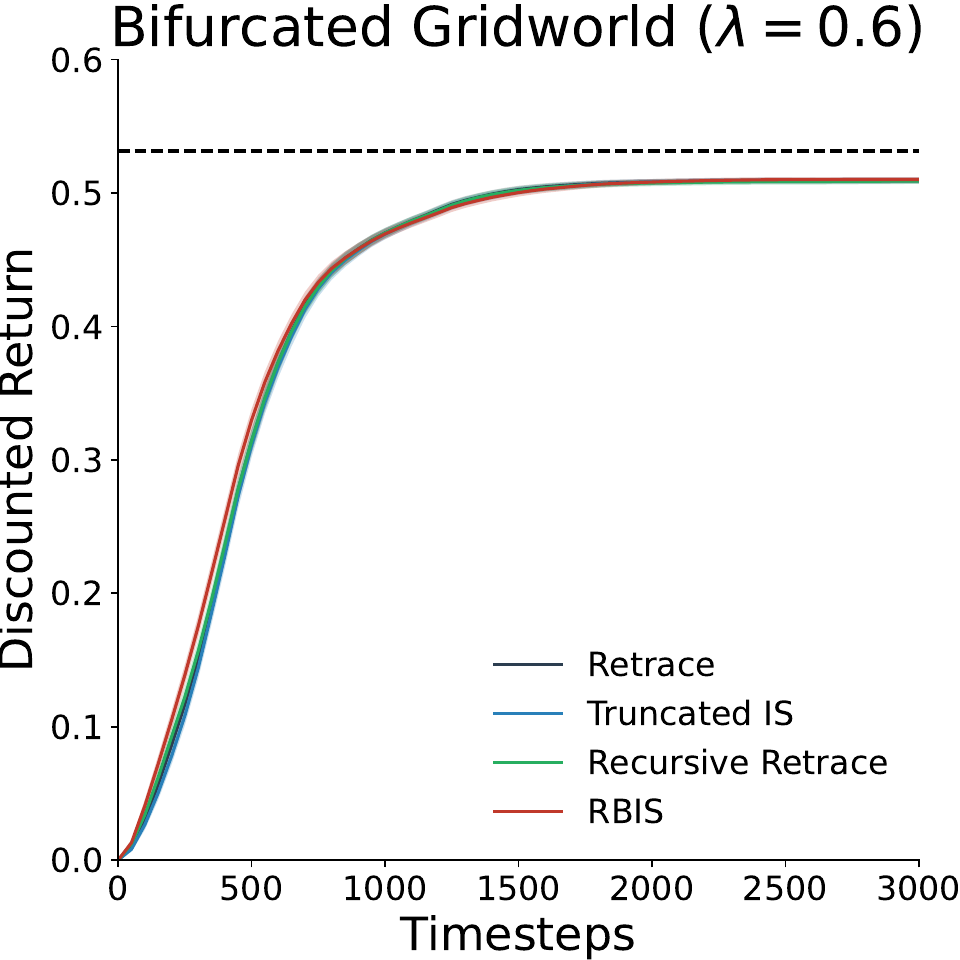}
    \hfill
    \includegraphics[width=0.3\textwidth]{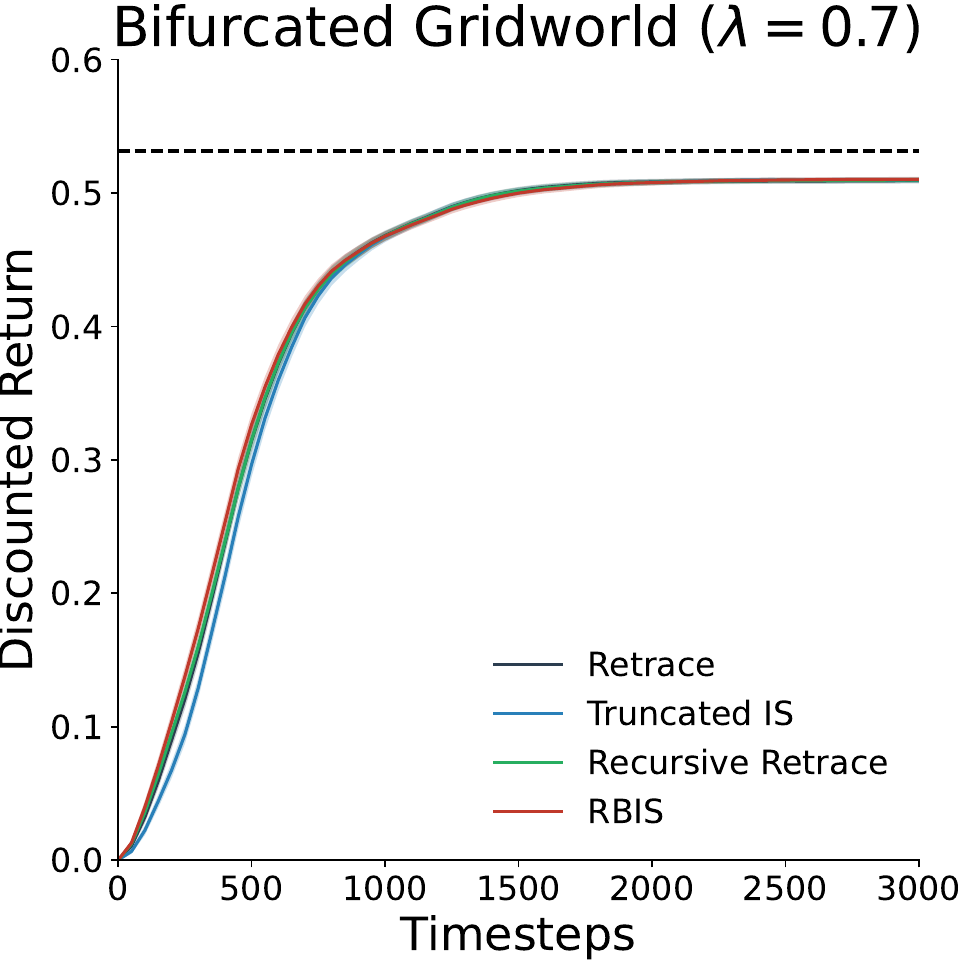}
    \hfill
    \includegraphics[width=0.3\textwidth]{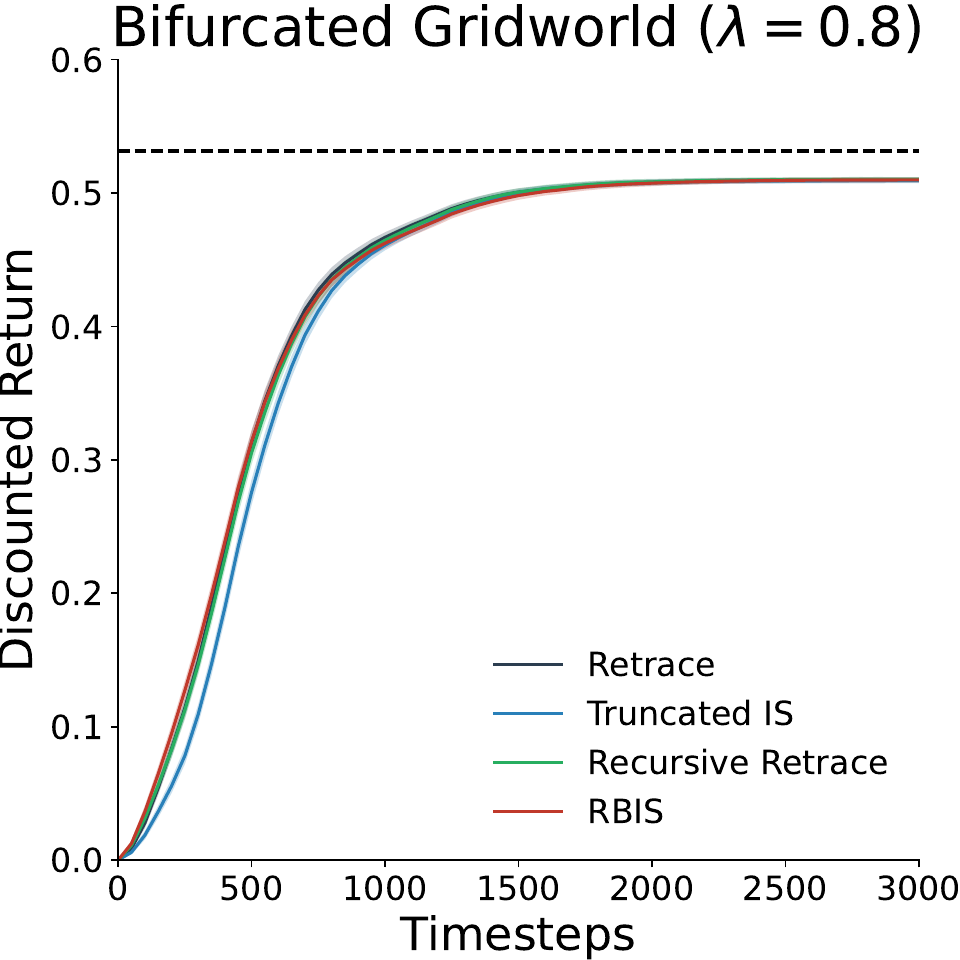}

    \vspace{0.2in}

    \hfill
    \includegraphics[width=0.3\textwidth]{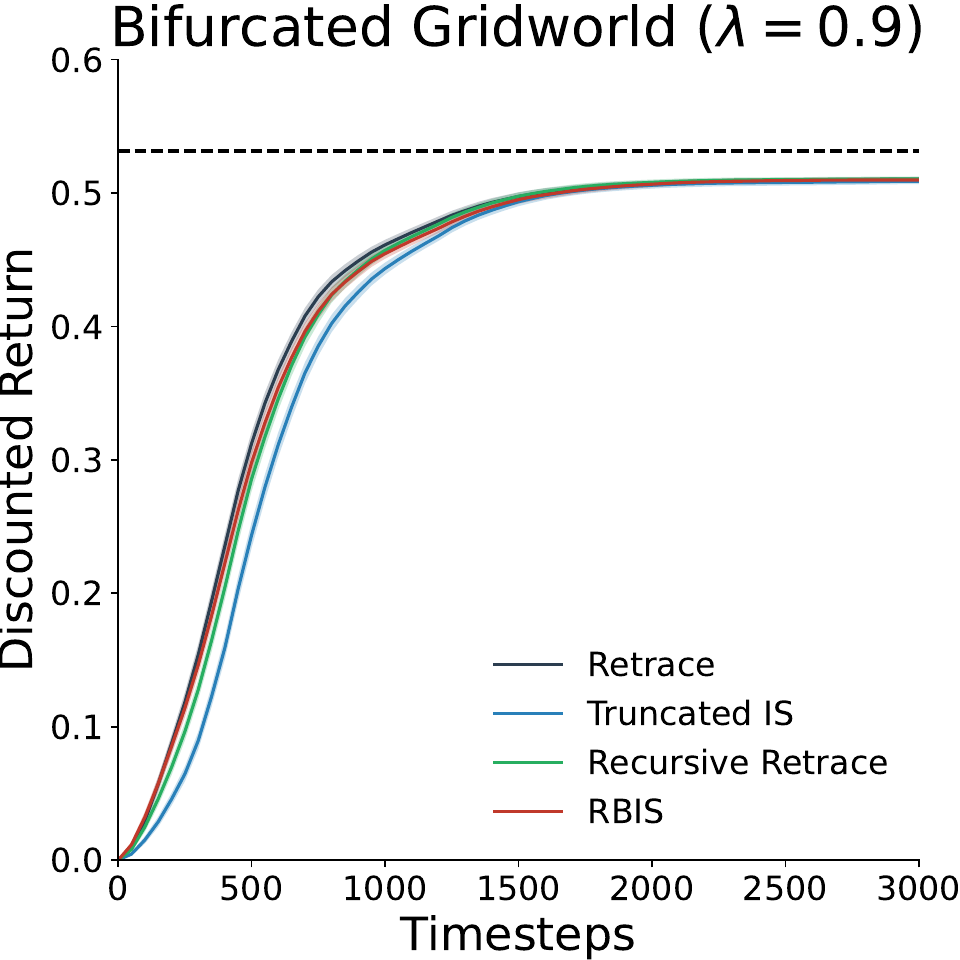}
    \hfill
    \includegraphics[width=0.3\textwidth]{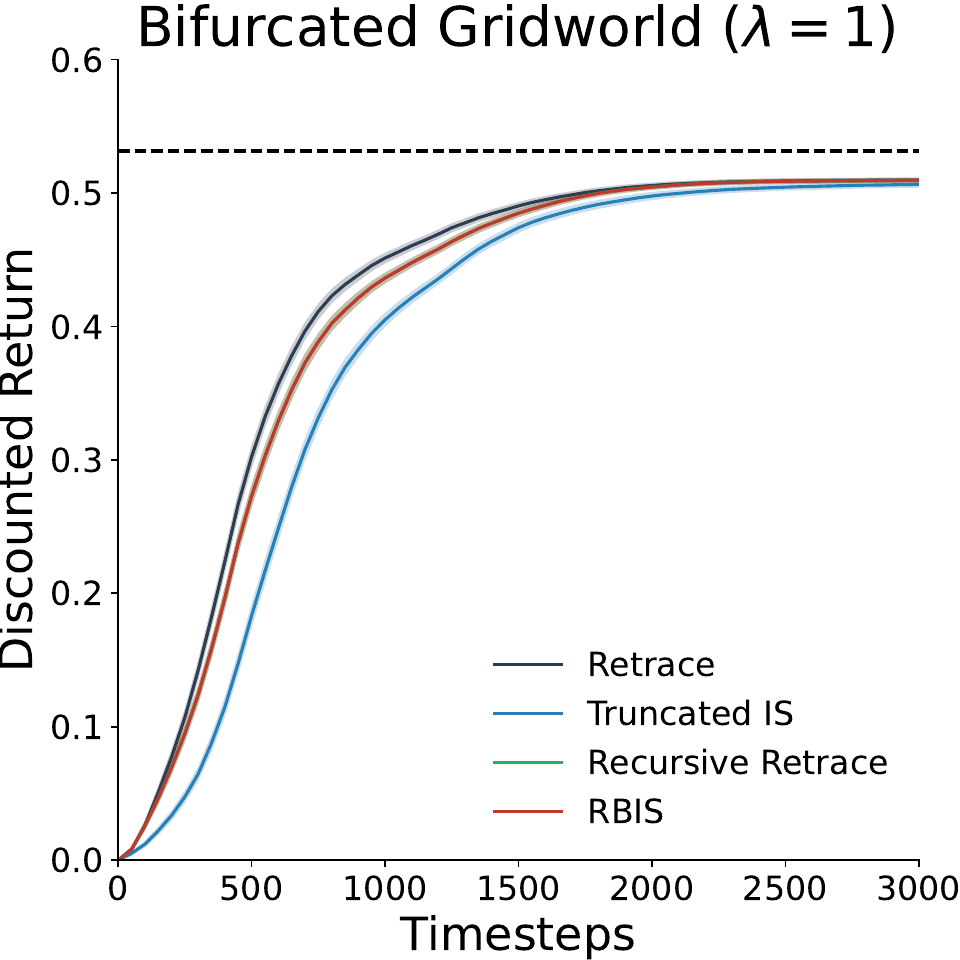}
    \hfill
    \caption{
        Learning curves for the $\lambda$-values we tested in the Bifurcated Gridworld environment.
        The dashed black line indicates the optimal discounted return for this problem.
    }
    \label{fig:learning_curves}
\end{figure}

%% file: figures/fig_additional_envs.tex
\begin{figure}[ht]
    \centering
    \includegraphics[width=0.3\textwidth]{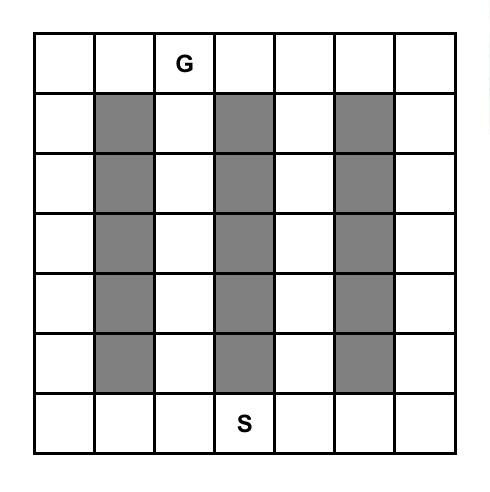}
    \hfill
    \includegraphics[width=0.65\textwidth]{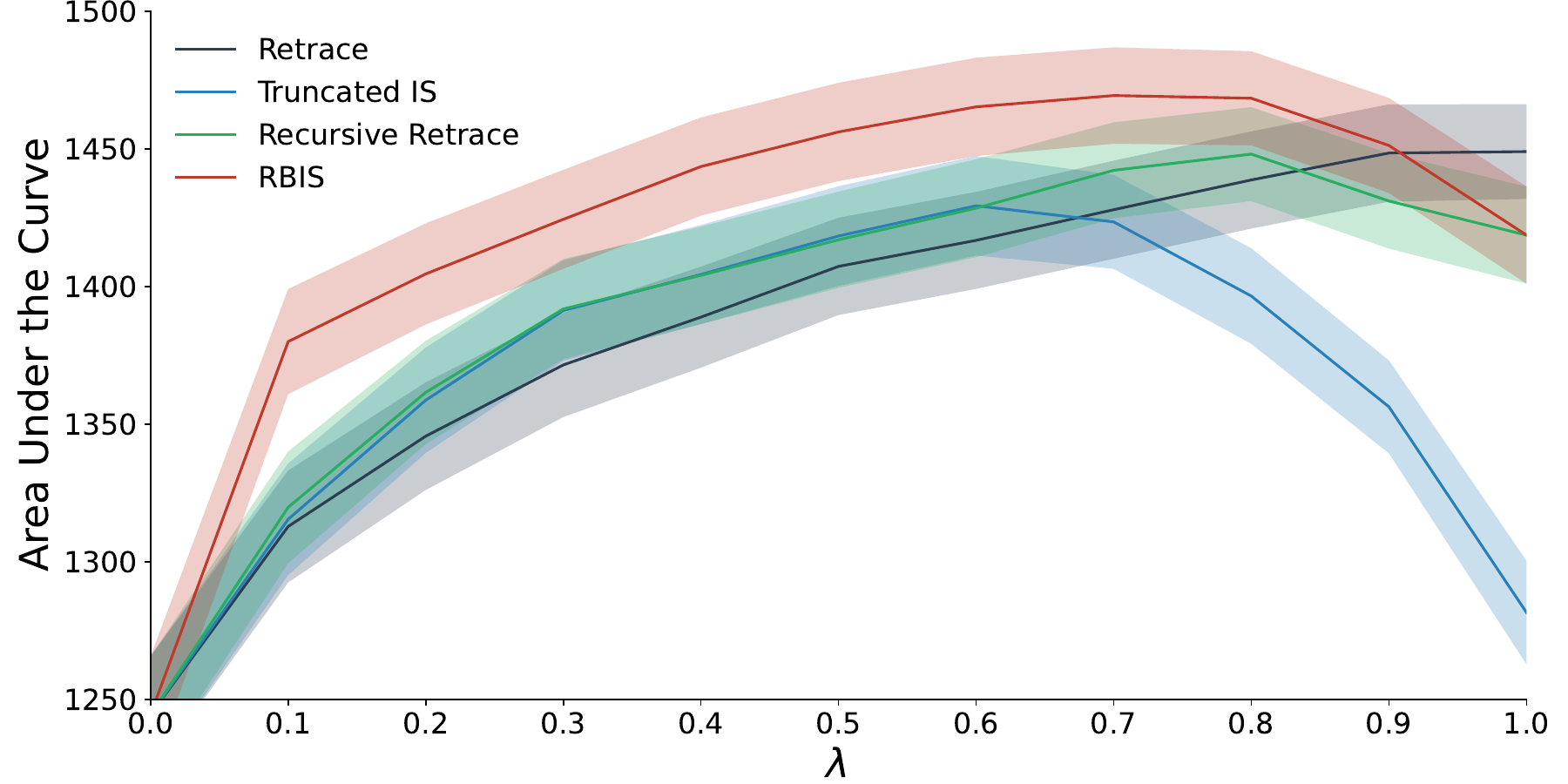}

    \vspace{0.2in}

    \includegraphics[width=0.3\textwidth]{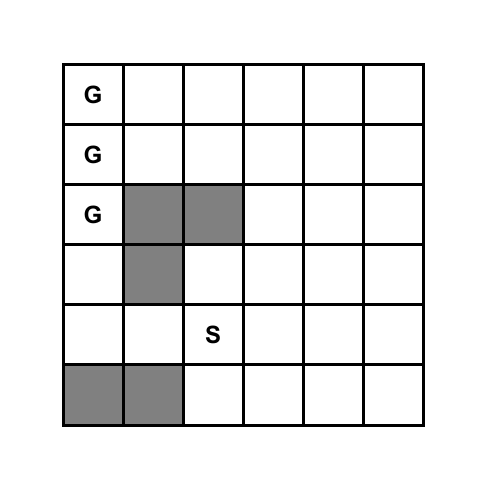}
    \hfill
    \includegraphics[width=0.65\textwidth]{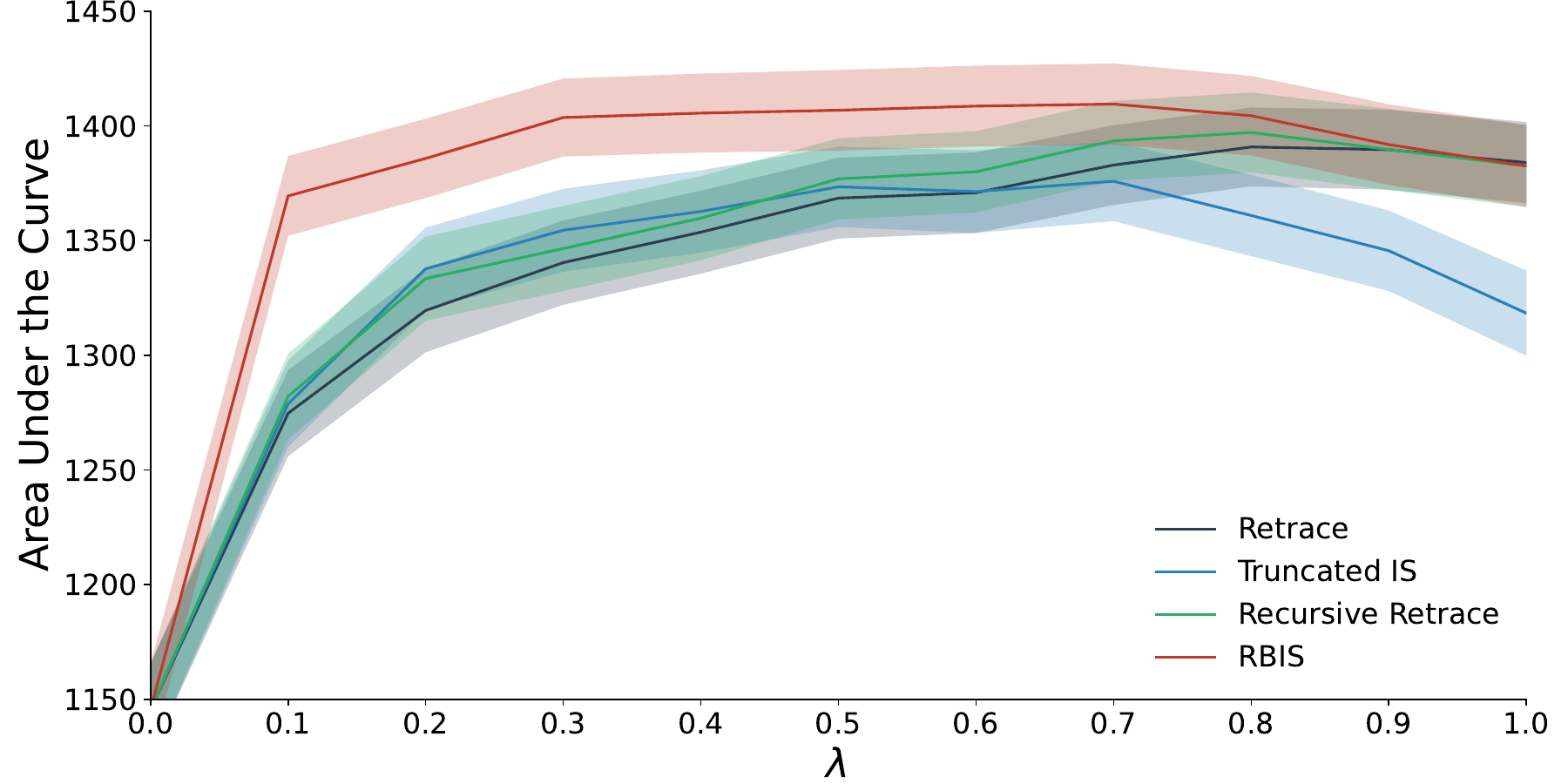}

    \vspace{0.2in}

    \includegraphics[width=0.3\textwidth]{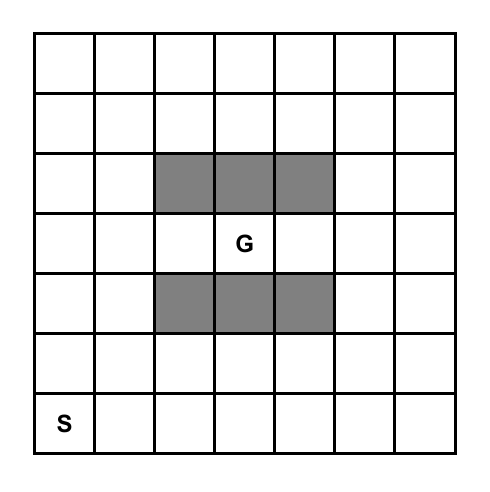}
    \hfill
    \includegraphics[width=0.65\textwidth]{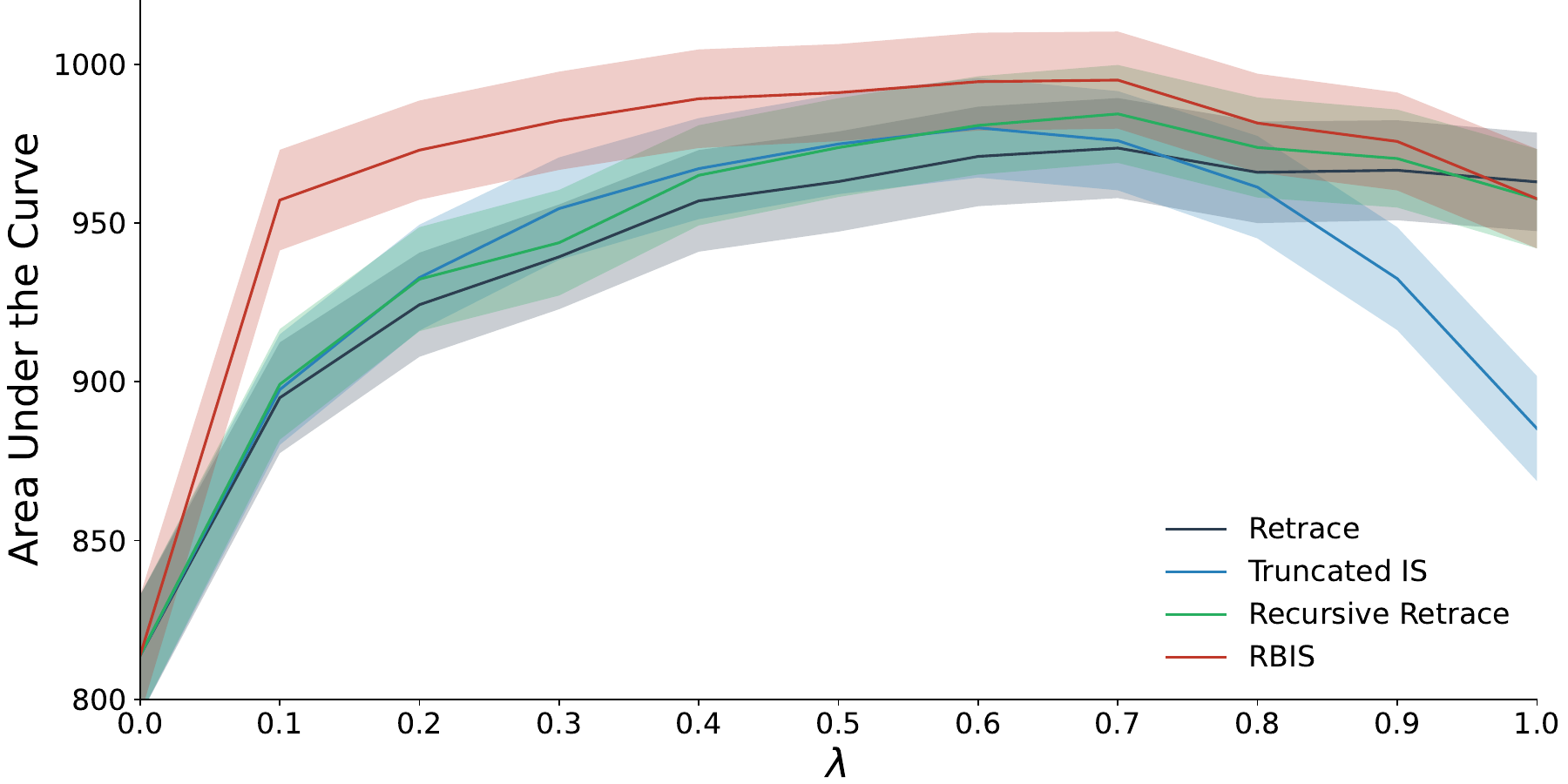}

    \caption{
        $\lambda$-sweeps conducted on three additional gridworld topologies.
        The experiment procedure was identical to that used in the creation of \Cref{fig:lambda_sweep}.
    }
    \label{fig:additional_envs}
\end{figure}

%% file: figures/algo_truncated_is.tex
\begin{algorithm}[H]
    \algospacing
    \caption{Truncated Importance Sampling}
    \label{algo:truncated_is}
    \begin{algorithmic}[1]
        \STATE {\bfseries Input:} value function $Q$, stepsize $\alpha \in (0,1]$
        \FOR{each episode}
            \STATE Reset environment and observe state $S_0$
            \STATE Reset dynamic array $Y$
            \REPEAT[for $t = 0, 1, 2, \dots$]
                \STATE Take action $A_t \sim \mu(\cdot|S_t)$, receive reward $R_t$, and observe next state $S_{t+1}$
                \STATE $\rho_t = \frac{\pi(A_t|S_t)}{\mu(A_t|S_t)}$
                \STATE $\delta_t = \begin{cases}
                        R_t - Q(S_t, A_t) &\text{if $S_{t+1}$ is terminal}\\
                        R_t - Q(S_t, A_t) + \gamma \sum_{a' \in \mathcal{A}} \pi(a'|S_{t+1}) Q(S_{t+1}, a') &\text{else}
                    \end{cases}$
                \FOR{$k = 0, \dots, t-1$}
                    \STATE $Y(k) \gets Y(k) \cdot \rho_t$
                \ENDFOR
                \STATE $Y(t) \gets 1$
                \FOR{$k = 0, \dots, t$}
                    \STATE $z \gets (\gamma \lambda)^{t-k} \min(1,~Y(k))$
                    \STATE $Q(S_k,A_k) \gets Q(S_k,A_k) + \alpha z \delta_t$
                \ENDFOR
            \UNTIL{$S_{t+1}$ is terminal}
        \ENDFOR
    \end{algorithmic}
\end{algorithm}

%% file: figures/algo_rbis.tex
\begin{algorithm}[H]
    \algospacing
    \caption{Recency-Bounded Importance Sampling (RBIS)}
    \label{algo:rbis}
    \begin{algorithmic}[1]
        \STATE {\bfseries Input:} value function $Q$, stepsize $\alpha \in (0,1]$
        \FOR{each episode}
            \STATE Reset environment and observe state $S_0$
            \STATE Reset dynamic array $Y$
            \REPEAT[for $t = 0, 1, 2, \dots$]
                \STATE Take action $A_t \sim \mu(\cdot|S_t)$, receive reward $R_t$, and observe next state $S_{t+1}$
                \STATE $\rho_t = \frac{\pi(A_t|S_t)}{\mu(A_t|S_t)}$
                \STATE $\delta_t = \begin{cases}
                        R_t - Q(S_t, A_t) &\text{if $S_{t+1}$ is terminal}\\
                        R_t - Q(S_t, A_t) + \gamma \sum_{a' \in \mathcal{A}} \pi(a'|S_{t+1}) Q(S_{t+1}, a') &\text{else}
                    \end{cases}$
                \FOR{$k = 0, \dots, t-1$}
                    \STATE \textcolor{blue}{$Y(k) \gets \min(\lambda^{t-k},~Y(k) \cdot \rho_t)$}
                \ENDFOR
                \STATE $Y(t) \gets 1$
                \FOR{$k = 0, \dots, t$}
                    \STATE \textcolor{blue}{$z \gets \gamma^{t-k} Y(k)$}
                    \STATE $Q(S_k,A_k) \gets Q(S_k,A_k) + \alpha z \delta_t$
                \ENDFOR
            \UNTIL{$S_{t+1}$ is terminal}
        \ENDFOR
    \end{algorithmic}
\end{algorithm}